\documentclass[
  aip,
  reprint,
]{revtex4-1}

\usepackage{graphicx}

\usepackage[caption=false]{subfig}

\usepackage{dcolumn}
\usepackage{bm}

\usepackage[T1]{fontenc}
\usepackage[utf8]{inputenc}
\usepackage{amsthm}

\usepackage{etoolbox}
\usepackage{mathtools}
\usepackage{array}
\usepackage{afterpage}
\usepackage{comment}

\usepackage{circuitikz}

\usepackage{algorithm}
\usepackage{algpseudocode}

\usepackage{expl3} 
\ExplSyntaxOn
\int_new:N \g_wc_total_int
\cs_new_protected:Npn \wc_reset: { \int_gset:Nn \g_wc_total_int {0} }
\cs_new_protected:Npn \wc_add:n #1 { \int_gadd:Nn \g_wc_total_int {#1} }
\cs_new:Npn \wc_total: { \int_use:N \g_wc_total_int }
\ExplSyntaxOff

\usepackage[colorlinks,linkcolor=blue]{hyperref}

\usepackage{amsmath,amssymb}
\usepackage{amsthm}        
\usepackage{mathtools}     
\usepackage{newtxtext,newtxmath}

\usepackage{bm}            
\usepackage{xcolor}        
\usepackage[normalem]{ulem}

\newtheorem{theorem}{Theorem}[section]

\newtheorem{proposition}[theorem]{Proposition}

\theoremstyle{definition}

\theoremstyle{remark}

\DeclareMathOperator{\diag}{diag}

\DeclarePairedDelimiter{\norm}{\lVert}{\rVert}

\newcommand{\mrN}{\mathrm{N}}
\newcommand{\mrE}{\mathrm{E}}
\newcommand{\mrEi}{\mathrm{E_i}}
\newcommand{\mrEo}{\mathrm{E_o}}
\newcommand{\mrC}{\mathrm{C}}

\newcommand{\E}{\mathbb{E}}

\newcommand{\bs}{\bm{s}}
\newcommand{\bi}{\bm{i}}
\newcommand{\bv}{\bm{v}}
\newcommand{\br}{\bm{r}}
\newcommand{\by}{\bm{y}}
\newcommand{\bz}{\bm{z}}

\newcommand{\bx}{\bm{x}}

\newcommand{\x}{\mathbf{x}}
\newcommand{\y}{\mathbf{y}}

\newcommand{\Sinput}{\bm{s}^{\text{(input)}}}
\newcommand{\Soutput}{\bm{s}^{\text{(output)}}}
\newcommand{\Shidden}{\bm{s}^{\text{(hidden)}}}

\newcommand{\Vinput}{\bm{v}^{\text{(input)}}}
\newcommand{\Voutput}{\bm{v}^{\text{(output)}}}
\newcommand{\Vhidden}{\bm{v}^{\text{(hidden)}}}

\newcommand{\sF}{\bm{s}_F}
\newcommand{\sC}{\bm{s}_C}
\newcommand{\vF}{\bm{v}_F}

\newcommand{\iF}{\bm{i}_F}
\newcommand{\iC}{\bm{i}_C}

\newcommand{\bhy}{\hat{\bm{y}}}
\newcommand{\bhv}{\hat{\bm{v}}}

\makeatletter
\def\@email#1#2{%
 \endgroup
 \patchcmd{\titleblock@produce}
  {\frontmatter@RRAPformat}
  {\frontmatter@RRAPformat{\produce@RRAP{*#1\href{mailto:#2}{#2}}}\frontmatter@RRAPformat}
  {}{}
}%
\makeatother

\begin{document}


\title{How to Train Your Resistive Network: Generalized Equilibrium Propagation and Analytical Learning}
\author{J. Lin}
\altaffiliation{These authors contributed equally to this work.}

\affiliation{Ming Hsieh Department of Electrical and Computer Engineering, University of Southern California, Los Angeles, CA, USA}
\affiliation{Center for Nonlinear Studies, Los Alamos National Laboratory, Los Alamos, New Mexico 87545, USA}
\affiliation{Theoretical Division (T4), Los Alamos National Laboratory, Los Alamos, New Mexico 87545, USA}

\author{A. Desai}
\altaffiliation{These authors contributed equally to this work.}
\affiliation{Center for Nonlinear Studies, Los Alamos National Laboratory, Los Alamos, New Mexico 87545, USA}
\affiliation{CAI-3, Los Alamos National Laboratory, Los Alamos, New Mexico 87545, USA}
\affiliation{ 
University of California, Berkeley, Berkeley, CA, 94720, USA
}
\author{F. Barrows}
\affiliation{Center for Nonlinear Studies, Los Alamos National Laboratory, Los Alamos, New Mexico 87545, USA}
\affiliation{Theoretical Division (T4), Los Alamos National Laboratory, Los Alamos, New Mexico 87545, USA}

 \email{fbarrows@lanl.gov}

\author{F. Caravelli}
\affiliation{Theoretical Division (T4), Los Alamos National Laboratory, Los Alamos, New Mexico 87545, USA}
\affiliation{Department of Physics, University of Pisa, Largo Bruno Pontecorvo 3, 56127 Pisa, Italy}

%

%

\email{caravelli@lanl.gov}


\begin{abstract}
    Machine learning is a powerful method of extracting meaning from data; unfortunately, current digital hardware is extremely energy-intensive. There is interest in an alternative analog computing implementation that could match the performance of traditional machine learning while being significantly more energy-efficient. However, it remains unclear how to train such analog computing systems while adhering to locality constraints imposed by the physical (as opposed to digital) nature of these systems. Local learning algorithms such as Equilibrium Propagation and Coupled Learning have been proposed to address this issue. In this paper, we develop an algorithm to exactly calculate gradients using a graph theoretic and analytical framework for Kirchhoff's laws. We also introduce Generalized Equilibrium Propagation, a framework encompassing a broad class of Hebbian learning algorithms, including Coupled Learning and Equilibrium Propagation, and show how our algorithm compares. We demonstrate our algorithm using numerical simulations and show that we can train resistor networks without the need for a replica or control over all edges.  
\end{abstract}

\maketitle

\section{Introduction}

Modern machine learning achieves impressive accuracy, but its energy cost is increasingly dominated by data movement
rather than arithmetic, motivating interest in \emph{physical} substrates that perform inference \emph{in situ} by
relaxing to steady states~\cite{Bourzac2024,Markovi2020,Jaeger_2023,Kaspar2021}. Resistive and in-memory electrical networks are especially appealing in
this context because they naturally implement low-power linear operations~\cite{SebastianEtAl2020InMemoryReview,Xia2019memristive,Yang2012, barrows_AID_2024, caravelli_review_2025}. 

A central obstacle to training physical systems are locality constraints: hardware exposes only local voltages and currents,
whereas standard gradient-based learning assumes access to global error signals. Two-phase learning rules, most
prominently Equilibrium Propagation (EP)~\cite{ScellierBengio2017EP,kendall_arxiv_2020}, address this mismatch by running two nearby
steady-state experiments under the same inputs in both a \emph{free} phase and a \emph{nudged} (weakly clamped) phase. A local
update is then formed from differences between the two equilibria, which in resistive circuits often reduces to a
``difference-of-squares'' rule. Related approaches, including Coupled Learning (CL)~\cite{RocksEtAlPRX2021SupervisedPhysicalNetworks,Stern_APLML_2024,Dillavou_PhysRevApplied_2022,Dillavou_PNAS_2024}, implement the
second phase by directly clamping the outputs rather than explicitly modifying the energy~\cite{Xie_NeurComp_2003,lecun2006tutorial}. While attractive, these
schemes require controlled nudging hardware and suffer from systematic estimation bias due to finite nudges ~\cite{LaborieuxEtAlScalingEP}. Some implementations further rely on replica (``twin'') networks for contrastive
readout ~\cite{Movellan_1991_contrastive,Hinton2002CD}.

In this letter we focus on the simplest setting of \emph{linear, memoryless} (but tunable) resistor networks: a clean
testbed for theory and co-design. Because the circuit admits a closed-form linear response, we show how to bypass
nudging entirely and compute \emph{exact} gradients with respect to edge resistances, in a form implementable using
local measurements and a small number of voltage- and current-mode circuit evaluations. To place this exact
``projector-based'' update in context, we also introduce \emph{Generalized Equilibrium Propagation} (GEP), a perturbative
viewpoint that unifies EP and CL by the order of their nudging perturbation, enabling a direct comparison between
two-phase estimators and the analytical circuit gradient. Our protocols use a single physical network (no replica is required) and
naturally accommodate partial actuation, sensing, and tunability across programmable-impedance platforms~\cite{Christensen_2022,Mehonic2020,WaserAono2007,Valov2011}.
In particular, they align with emerging ``learning machines'' viewpoints where hardware, dynamics, and learning rules are co-designed rather than abstracted away.
Section~\ref{sec:physical-tp} gives a minimal physical derivation of GEP via linear response theory. We then
specialize to resistor networks, introduce the response operator formulation, and present both contrastive two-phase and projector-based analytical-gradient training rules. Section~\ref{sec:physical-learning} discusses standard two-phase learning in electrical circuits and introduces the analytical (projector-based) gradient method, which is the main result of this manuscript.  Section~\ref{sec:applications} applies these learning algorithms to both the classification and regression settings, demonstrating the difference between the analytical method and a representative two phase method. Conclusions follow.

\section{Learning Rules for Physical Systems}\label{sec:physical-learning}
\subsection{Two-phase from linear response}
\label{sec:physical-tp}
Let us begin with a definition of  \emph{two-phase learning}. In this context, we mean any learning protocol that estimates a parameter update by comparing two nearby
steady states of the \emph{same} (or a replicated) physical system under the same inputs using: (i) a \emph{free} phase, in which the system relaxes naturally, and (ii) a weakly \emph{nudged} (or \emph{clamped}) phase, in which a small external ``training
field'' is applied to encode the target. In this sense, as we show here, two-phase rules are concrete realizations of
\emph{linear-response} ideas from nonequilibrium statistical mechanics: the parameter update is inferred from how the
steady state shifts under a weak perturbation~\cite{CallenWelton1951FDT,Kubo1957IrreversibleProcessesI,Onsager1931ReciprocalI}.

These two-phase learning rules can be understood through the analytical framework provided by GEP. The full derivation of GEP is provided in Appendices \ref{app:EPder} and \ref{app:gep}, and we provide a brief overview here. We consider an input-clamped physical system with state $y\in\mathbb{R}^n$ and energy (or free energy at finite temperature)
$E_\theta(y)$ depending on tunable parameters $\theta$. A broad class of dissipative dynamics can be written as a
gradient flow
\begin{equation}
\dot y=-\Gamma \nabla_y E_\theta(y),\qquad \Gamma\succ 0,
\label{eq:gradflow-main-short}
\end{equation}
so relaxation decreases $E_\theta$ and converges to a stable equilibrium $y^0(\theta)$ with $\nabla_y E_\theta(y^0)=0$
(see, e.g., standard references on gradient flows~\cite{AmbrosioGigliSavare2008GradientFlows}). This \emph{free phase}
is simply the steady state reached under the imposed inputs.

To incorporate targets, we apply a small \emph{training field} (a weak clamp on outputs) to induce a nudged energy:
\begin{eqnarray}
F(\beta, \theta, y)&=&E_\theta(y)+\beta C(\theta, y)
\label{eq:nudged-energy-main-short}
\end{eqnarray}
for an effective cost function $C$. We let $y^\beta(\theta)$ be the resulting \emph{nudged/clamped phase} equilibrium,
$\nabla_y F(\beta, \theta ,y^\beta)=0$. In the linear-response regime $\beta\to 0$, the shift $y^\beta-y^0$ is $O(\beta)$
and encodes the susceptibility of the steady state to the applied field~\cite{Kubo1957IrreversibleProcessesI,CallenWelton1951FDT}.
Defining the objective function as $J(\theta):=C(y^0(\theta))$, one obtains the Equilibrium Propagation identity~\cite{ScellierBengio2017EP}
\begin{equation}
\partial_\theta J(\theta)
=
\lim_{\beta\to 0}\frac{1}{\beta}
\Big[\partial_\theta F(\beta, \theta, y^\beta)-\partial_\theta F(0, \theta, y^0)\Big].
\label{eq:ep-identity-main}
\end{equation}

Equation~\eqref{eq:ep-identity-main} makes explicit why two-phase learning is ``physical'': the gradient is recovered by
comparing two relaxation experiments, free evolution to $y^0$ and weakly clamped evolution to $y^\beta$, and forming
a local difference of energy derivatives evaluated at the two steady states~\cite{ScellierBengio2017EP}.

Equilibrium Propagation (EP) and Coupled Learning (CL) are both two-phase methods: they compare a free steady state to a
nearby nudged steady state and extract an update from their difference~\cite{ScellierBengio2017EP,RocksEtAlPRX2021SupervisedPhysicalNetworks,Stern_2024}. Mathematically, both EP and CL have nudged energies of the form
\begin{align}
    F(\beta,\theta,y)=E_\theta(y)+n(\beta,\theta,y)
    \label{eq:GEP_nudged_energy}
\end{align}
The primary difference between the two learning rules is that EP uses an explicit \emph{linear} energy nudge, while CL nudges by \emph{output clamping} (a state shift inside the same energy). In other words, the key distinction is the \emph{perturbative order} of the nudging perturbation at the free equilibrium point $y^0$:
\begin{eqnarray}
n_{EP}(\beta,\theta,y^0)&=&\beta\,C(\theta,y^0)=O(\beta)
\\
n_{CL}(\beta,\theta,y^0)&=&E_\theta\!\big(y^0 + \epsilon \big)-E_\theta(y^0)=O(\beta^2)
\end{eqnarray}

where $\epsilon$ is an $O(\beta)$ error vector (see Appendix \ref{app:epcl} for details). In EP, the nudged energy differs from the free energy by an $O(\beta)$ term at $y^0$. In CL the leading change is $O(\beta^2)$.~\cite{RocksEtAlPRX2021SupervisedPhysicalNetworks,Stern_2024}.

We can capture both cases within a single perturbative framework by assuming that that $n(\beta,\theta,y)=O(\beta^k)$ to leading order for some integer $k\ge 1$. Similarly, we assume that the free and nudged equilibria $y^0(\theta)$ and $y^\beta(\theta)$ satisfy $\|y^\beta-y^0\|=O(\beta^k)$. If we have a generalized objective function of the form
\begin{align}
    \ell(\theta)=\frac{1}{k!}\partial_\beta^k n(0,\theta,y^0),
\end{align}
the GEP identity states that the leading-order two-phase finite difference in parameter derivatives is proportional to $\partial_{\theta} \ell(\theta)$:
\begin{equation}
\lim_{\beta\to 0}\frac{1}{\beta^k}
\Big[\partial_\theta F(\beta,\theta,y^\beta)-\partial_\theta F(0,\theta,y^0)\Big]
=
\partial_{\theta} \ell(\theta),
\label{eq:gep_identity}
\end{equation}
This identity is rigorously derived in Appendix \ref{app:gep}. Note that EP emerges in the case when $k=1$ (linear nudge) ~\cite{ScellierBengio2017EP} while CL corresponds to $k=2$ (quadratic
nudge)~\cite{RocksEtAlPRX2021SupervisedPhysicalNetworks,Stern_2024}, placing their two-phase estimators on a common
footing for comparison with the exact resistor-network gradients developed next.

\subsection{Learning with Electrical Circuits}
\label{sec:circuits}

We now specialize to \emph{passive linear resistor networks}, where the steady state admits an explicit linear
input-output map~\cite{guillemin1953introductory,Chung1996,Zegarac2019,Lin2025}. For full generality, we assume the circuit to be associated with a connected 
graph $\mathcal{G} = (\mathcal{V}, \mathcal{E})$ with $\mathrm N = |\mathcal{V}|$ nodes and $\mrE = |\mathcal{E}|$ edges ($\mrEi$ input edges and $\mrEo$ output edges); each edge $e$ contains an ideal voltage source $s_e$ in series with a resistor of resistance $r_e>0$ (conductance $g_e=r_e^{-1}$). We collect edgewise quantities into vectors
$\bs,\bi,\bv\in\mathbb{R}^\mrE$, where $\bs$ are imposed source voltages, $\bi$ are steady-state edge currents, and $\bv$
are Ohmic voltage drops across resistors (so $v_e=r_e i_e$). Let
\begin{equation}
G=\diag(g_e)\succ 0,
\qquad
R=\diag(r_e)=G^{-1}.
\end{equation}

To obtain a closed-form circuit response, we use a standard cycle-space formulation of Kirchhoff constraints~\cite{guillemin1953introductory,Chung1996,Caravelli_2017b,Caravelli_2021,Barrows2024}
and fix a circuit orientation. Let $A\in\mathbb{R}^{\mathrm \mrC \times \mrE}$ be a cycle matrix spanning the
$\mrC = \mrE- \mrN + 1$ fundamental cycles. Solving Kirchhoff constraints together with dissipation minimization (Thomson/Dirichlet
principles~\cite{doyle_arxiv_2000,barrows_arxiv-PEDS-2025}) yields a linear edge-space map (derivation in Appendix~\ref{app:projectors})
\begin{equation}
\bv=-\Omega_{A/R}\bs,
\qquad
\Omega_{A/R}:=R\,A^\top(AR A^\top)^{-1}A,
\label{eq:omega_voltage}
\end{equation}
which can be interpreted as a (weighted) cycle-space projector written in edge variables~\cite{Caravelli2017_loc,Zegarac2019,Lin2025}. 
We identify the $R$-weighted cycle-space projector $\Omega_{A/R}$ as the central learning operator in passive resistor networks and derive an exact, physically interpretable gradient for training edge conductances. 
Currents follow from Ohm's law,
\begin{equation}
\bi=R^{-1}\bv=-R^{-1}\Omega_{A/R}\bs.
\label{eq:omega_current}
\end{equation}

Operationally, a voltage-mode experiment implements $\bs\mapsto\bv$ (i.e., left multiplication by $-\Omega_{A/R}$),
while reciprocal/current-mode manipulations provide access to $\Omega_{A/R}^\top$ via standard reciprocity properties
of passive linear networks~\cite{guillemin1953introductory,Zegarac2019}. This transpose-access primitive is the key
ingredient exploited by the projector-based learning rule below, connecting circuit measurements to two-phase
gradient estimators in energy-based learning. 

\begin{figure*}[ht]
    \centering
    \includegraphics[width=\textwidth]{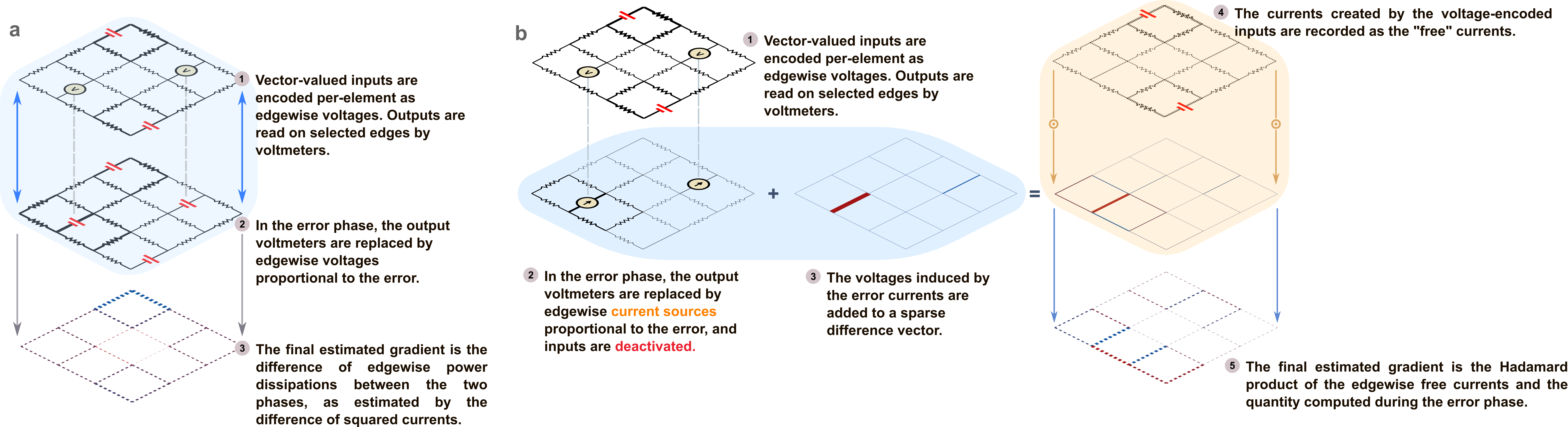}
    \caption{A schematic comparison between (a) two-phase gradient estimation and (b) projector-based gradient
    estimation. The two-phase estimate uses purely voltage-mode biasing and measurement to obtain a gradient proxy,
    while the projector-based estimator mixes voltage- and current-mode manipulations. In two-phase learning, both
    input and error signals are encoded as voltage sources; in projector-based learning, error signals are encoded via
    current sources to access an adjoint projector action.}
    \label{fig:both_schemes}
\end{figure*}

\subsection{Two learning paradigms: contrastive two-phase vs.\ projector gradients}
\label{sec:two_paradigms}

\textit{Vanilla two-phase learning.} {We now consider our \textit{vanilla} two-phase learning algorithm (EP-style).}
Let $P_o\in\{0,1\}^{\mrE\times \mrEo}$ be a linear operator selecting the $\mathrm {E_o}$ output edges and define $\by \in\mathbb{R}^{\mathrm{E_o}}$ to be the desired target readout.

Equilibrium Propagation introduces a weak output clamp by augmenting the (dissipation) energy with
$\frac{\beta}{2}\|P_o^\top \bv-\by\|^2$ and compares a free equilibrium to a nudged equilibrium.
The resulting local two-phase update is a difference of squares; in resistance form,
\begin{equation}
\frac{\partial J}{\partial \br}
=
\lim_{\beta\to 0}\frac{1}{2\beta}\Big[\iC^{\odot 2}-\iF^{\odot 2}\Big],
\label{eq:ep_resistance_update}
\end{equation}
where $\bi_F$ and $\bi_C$ are the free- and clamped-phase steady-state currents.

\textit{$\Omega$-based (projector) learning.}
Because the circuit is linear, we can instead differentiate the \emph{closed-form} map
\eqref{eq:omega_voltage}--\eqref{eq:omega_current} and obtain an \emph{analytical} gradient with respect to $\br$.
Crucially, the resulting expressions can be implemented physically using a small number of voltage- and current-mode
experiments that realize $\Omega_{A/R}$ and $\Omega_{A/R}^\top$. This avoids finite-$\beta$ bias and does not require
engineering a nudged energy, in contrast to vanilla two-phase learning.


We view $\bv(\br;\bs)=-\Omega_{A/R}\bs$ as a function of $\br$ with $\bs$ held fixed. Differentiating
$\bv(\br;\bs)=-R A^\top(AR A^\top)^{-1}A\bs$ yields the Jacobian
\begin{align}
J_{\bv}(\br)
:=\frac{\partial \bv}{\partial \br}
&=
\diag(\bi)\;-\;\Omega_{A/R}\,\diag(\bi)
\nonumber\\
&=
(I-\Omega_{A/R})\,\diag(\bi),
\label{eq:jacobian_v_r}
\end{align}
where $\bi$ is the steady-state current vector \eqref{eq:omega_current}. For the least-squares loss
$\mathcal{L}(\bv,\by)=\frac12\|P_o \bv- \by\|^2$, the chain rule gives
\begin{equation}
\nabla_{\br}\mathcal{L}
=
J_{\bv}(\br)^\top P_o^{\top} (\bv-\by)
=
\diag(\bi)\,(I-\Omega_{A/R}^\top)\,P_o^{\top} (P_o \bv-\by).
\label{eq:proj_grad_ls}
\end{equation}
(Analogous expressions hold for other losses, e.g.\ hinge-style objectives; we return to classification variants in
Section~\ref{sec:classification}.)

With $J_v(r)=(I-\Omega_{A/R})\,\mathrm{diag}(i)$: (i) the circuit implements the linear map $v=-\Omega_{A/R}s$ and $i=-R^{-1}\Omega_{A/R}s$; (ii) a key structural fact is that $\Omega_{A/R}$ is an oblique projector: $\Omega_{A/R}^2=\Omega_{A/R}$; (iii) for least squares, the local first-order condition is $\mathrm{diag}(i)\,(I-\Omega_{A/R}^{\top})(v-v_T)=0$; (iv) the input–output voltages satisfy
\begin{align}
\|v_{\mathrm{(output)}}\|
\le \|s_{\mathrm{(input)}}\|_2\sqrt{R_{\max}/R_{\min}},
\end{align}
where $R_{\max}$ and $R_{\min}$ denote the largest and smallest diagonal entries of $R$, respectively. For proof, see Appendices \ref{app:projectors}-\ref{sec:bounds}.

In the linear-response limit, the two-phase “difference-of-squares” update converges to
\begin{align}
\lim_{\beta\to 0}\frac{1}{2\beta}\bigl(\iC^{\odot 2}-\iF^{\odot 2}\bigr)
= \diag(\iF)  R^{-1} \Omega_{A/R} P_o^{\top} (\by - P_o \vF)
\end{align} See Appendix \ref{app:twophaseelectric} for the full proof.

\begin{figure*}[ht!]
    \begin{minipage}{0.48\textwidth}
    \begin{algorithm}[H]
    \caption{Two-phase gradient estimation}
    \begin{algorithmic}[1]
        \State $S_F(\bx) \gets \gamma P_i\bx$.
        \State With voltage-mode biasing, compute $\bv = -\Omega_{A/R}S_F(\bx)$.
        \State \textbf{Measure} the ``free" edgewise currents, $\bi_F$.
        \State \textbf{Measure} output edges to compute $\bhy = P_o^{\top}\bv$.
        \State Compute the error vector $\epsilon = \beta(\bhy - \by)$.
        \State $S_C(\bx, \epsilon) \gets S_F(\bx) + P_o\epsilon$.
        \State Apply $S_C$: \textbf{Measure} the ``nudge" currents, $\bi_C$.
        \State Apply (\textit{estimated}) gradient $\Delta r = \bi_F^{\odot 2} - \bi_C^{\odot 2}$.
    \end{algorithmic}
    \end{algorithm}
    \end{minipage}
    \hfill
    \begin{minipage}{0.48\textwidth}
    \begin{algorithm}[H]
    \caption{$\Omega$-gradient estimation}
    \begin{algorithmic}[1]
        \State $S_F(\bx) \gets \gamma P_i\bx$.
        \State With voltage-mode biasing, compute $\bv = -\Omega_{A/R}S_F(\bx)$.
        \State Store the ``free" edgewise currents, $\bi_F$.
        \State \textbf{Measure} output edges to compute $\bhy = P_o \bv$.
        \State Compute the error vector $\epsilon = \eta(\bhy - \by)$.
        \State Compute $\bi_\epsilon = -\Omega_{A/R}^{\top}\epsilon$ via current-mode bias.
        \State Compute $\Delta = \epsilon + \bi_\epsilon$.
        \State Apply (\textit{analytical}) gradient $\nabla_r\mathcal{L}(\bhy,\by) = \bi_F \odot\,\Delta$.
    \end{algorithmic}
    \end{algorithm}
    \end{minipage}
    \caption{The two gradient estimation schemes discussed in this work used to modify the resistive states. Both methods train the mapping $\bhy = -P_o\Omega_{A/R}S_F(\bx)$ as computed by a given resistive mesh. On the left algorithm, multiple measurements are required to compute the estimated gradient. In the $\Omega$ gradient approach developed in this manuscript, only the output measurement is required.}
    \label{fig:schemes_code}
\end{figure*}

\subsection{Circuits as tunable input--output maps}
\label{sec:circuit_tunable_main}

To define inference and training experiments, we choose $i$ input edges to actuate and $o$ output edges to read out.
Let $P_i\in\{0,1\}^{\mrE \times \mrEi}$ and $P_o\in\{0,1\}^{\mrE \times \mrEo}$ be the corresponding selector matrices. In a free
(voltage-mode) inference step, an input $\bx\in\mathbb{R}^{\mrEi}$ is encoded as an edge-source pattern
\begin{equation}
\sF(\bx)=\gamma P_i^{\top}\bx,
\end{equation}
yielding the steady-state ohmic drops and readout
\begin{equation}
\bv_F=-\Omega_{A/R}\sF(\bx),
\qquad
\hat{\by}=P_o^\top \bv_F.
\end{equation}
Training adjusts $R$ so that $\hat{\by}\approx\by$. Vanilla two-phase learning introduces a second (clamped) experiment by adding
a small error-proportional source perturbation on the output edges; equivalently,
\begin{equation}
\sC(\bx,\by)=\sF(\bx)+P_o\,\beta(\hat{\by}-\by),
\end{equation}
so that free and clamped equilibria differ only on the output-edge sources. These two experiments lead to the
contrastive update \eqref{eq:ep_resistance_update}. In contrast, the projector gradient
\eqref{eq:proj_grad_ls} uses the same free inference quantities together with physical realizations of
$\Omega_{A/R}^\top$ to compute an analytical update without nudging. For further details on edge reordering, block
partitions, and equivalent source definitions, see Appendix~\ref{app:circuit_tunable}.

\begin{figure*}[ht!]
    \centering
    \includegraphics[width=\textwidth]{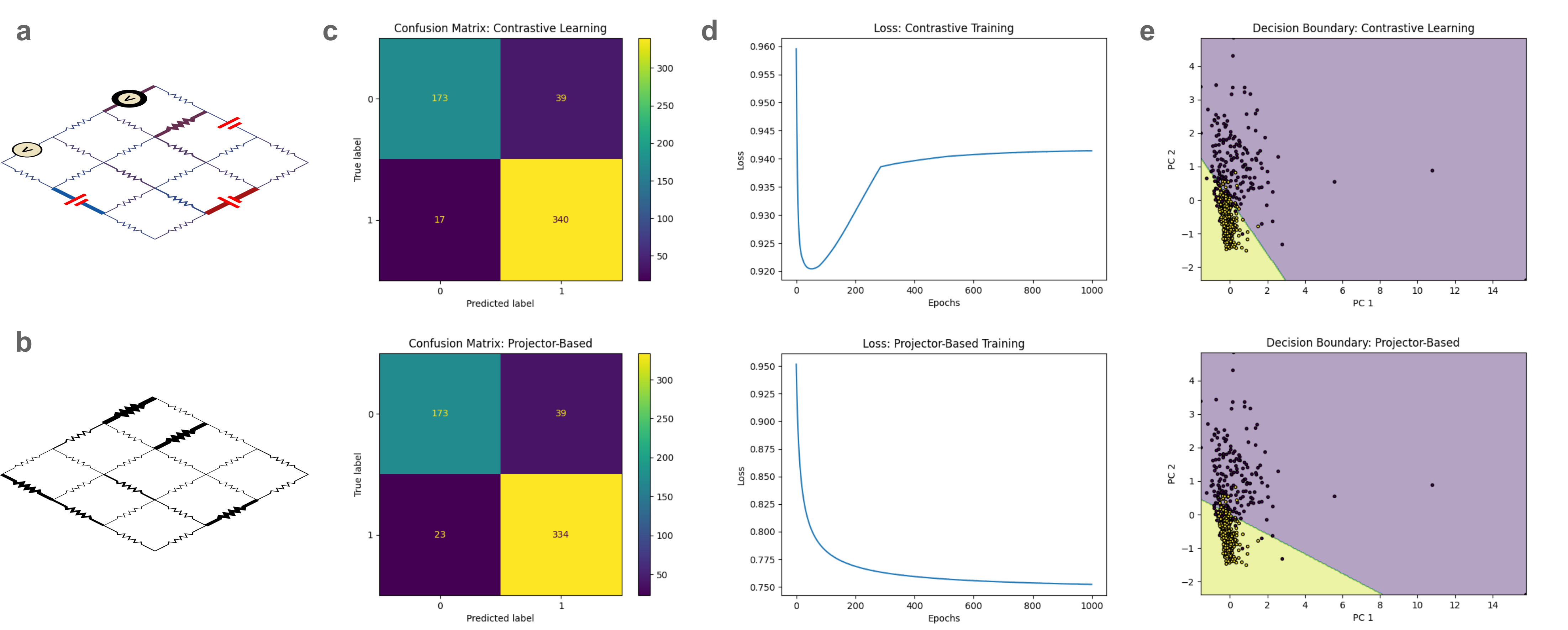}
    \caption{Training results on the Wisconsin breast cancer classification dataset: both methods train the mesh
    classifier to reach 90\% accuracy. (a) A schematic, with example voltages, of the circuit during inference: the
    30-dimensional data is projected to 3 dimensions using PCA. The output class is read as the maximum over two
    voltmeters. (b) The learned configuration of edge resistances. (c) Confusion matrices for each gradient method.
    (d) Loss curves for both methods. While comparable accuracy is achieved by both classifiers, training with
    contrastive learning displays greater instability. (e) The decision boundary learned by each classifier.}
    \label{fig:cancer_classifier}
\end{figure*}

In Figure \ref{fig:both_schemes} we present a schematic of a circuit implementation of two phase learning and the  projector based analytical gradient method.

\subsection{Learning Recipes}\label{sec:rec}

We summarize the two training protocols in Fig.~\ref{fig:schemes_code} using the same notation as above. Let
$P_i\in\{0,1\}^{E\times s}$ select the $\mrEi$ input edges and $P_o\in\{0,1\}^{\mrE \times \mrEo}$ select the $\mrEo$ output edges.
Given an input $\bx\in\mathbb{R}^{\mrEi}$, we encode it as an edge-source pattern
\begin{equation}
S_F(\bx)=\gamma P_i^{\top}\bx,
\end{equation}
where $\gamma$ is a fixed input gain. Applying $S_F$ in voltage mode produces the steady-state ohmic drops
$\bv=-\Omega_{A/R}S_F(\bx)$, and the model output is obtained by reading out the selected output edges,
\begin{equation}
\bhy=P_o \bv.
\end{equation}
Thus, a single ``inference'' experiment consists of imposing $S_F(\bx)$, then measuring $\bhy$.

Two-phase learning performs a gradient estimate by comparing this free experiment to a second, weakly clamped experiment.
After measuring $\bhy$, we form an error signal
\begin{equation}
\epsilon=\beta(\bhy-\by),
\end{equation}
and inject it back onto the output edges as an additional source, defining the clamped bias pattern
\begin{equation}
S_C(\bx,\by)=S_F(\bx)+P_o\epsilon .
\end{equation}
We then run the circuit a second time under $S_C$ and measure the resulting steady-state edge currents $\bi_C$; we also
store the free-phase currents $\bi_F$ from the first run. The contrastive (two-phase) update is local and edgewise:
each resistance is driven by the difference between the squared currents observed in the free and clamped phases,
\begin{equation}
\Delta\br \ = \bi_F^{\odot 2}-\bi_C^{\odot 2}
\end{equation}
In other words, the clamped experiment slightly perturbs the network toward the target, and the change in dissipated
power on each edge (proportional to $i_e^2 r_e$) provides the learning signal.

\begin{figure}[ht!]
  {\raggedleft
    (a)\includegraphics[width=0.99\linewidth]{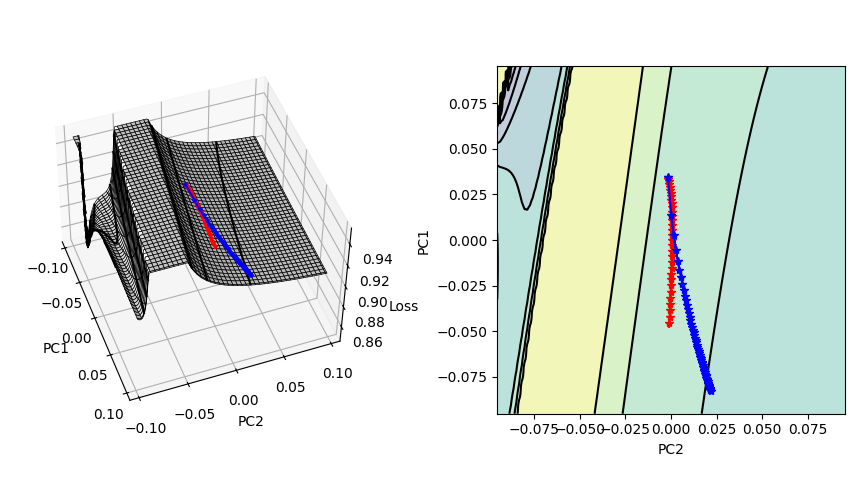}\par
    (b)\includegraphics[width=0.99\linewidth]{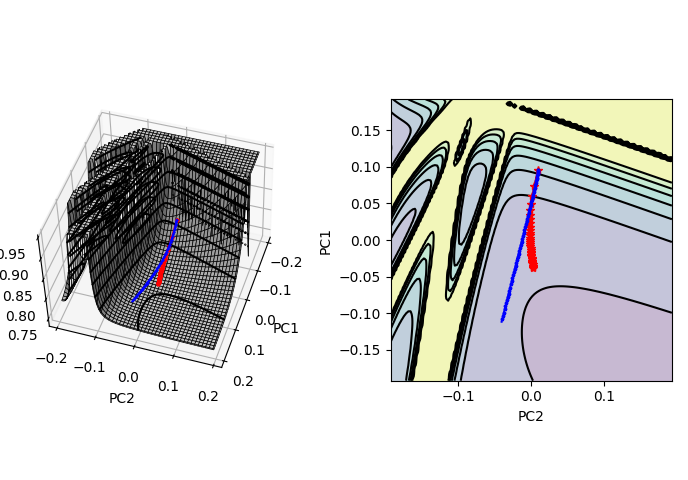}\par
  }
      \caption{A comparison between the loss landscapes ``seen'' by the (a) two-phase and (b) projector-based
    gradients on the Wisconsin breast cancer task. The red trajectory indicates the weight evolution under
    projector-based training, while the blue trajectory is for two-phase learning. In both cases, the trajectories
    are plotted such that the initial state is towards the top, and the final state is towards the bottom.}
  \label{fig:loss_landscapes}
\end{figure}

The $\Omega$-based method replaces the second voltage-mode (clamped) experiment by a single adjoint probe that
implements $\Omega_{A/R}^\top$. After the free inference run and the readout of $\bhy$, we again form a small error
vector
$\epsilon=\alpha(\bhy-\by)$,
but instead of adding $\epsilon$ as an output voltage source, we apply it in current mode to realize the adjoint action
$\bv_\epsilon=-\Omega_{A/R}^\top\epsilon$. We then compute
\begin{equation}
\Delta=\epsilon + \bv_\epsilon,
\end{equation}
and combine this with the stored free-phase currents to obtain an edgewise gradient estimate
\begin{align}
\nabla_{\br}\mathcal{L}(\bhy,\by) &= \diag(\mathbf{i}) (I - \Omega_{A/R}^{T}) (v - v_T)  \nonumber
\\
    &= \bi_F\odot \Delta.
\label{eq:ProjectorAnalyticalGradient_MainText}
\end{align}
Operationally, this scheme uses two physical primitives: a voltage-mode experiment to realize $-\Omega_{A/R}$ (the free
map $\bs\mapsto\bv$) and a reciprocal current-mode experiment to realize $-\Omega_{A/R}^\top$ (a voltage mode realization of $-\Omega_{A/R}^\top$ is described in Appendix \ref{app:circuit_tunable}). The resulting update is
still local as it multiplies a measured current on each edge by a locally computed $\Delta$ on that edge, but it avoids a
finite-$\beta$ clamping step and does not require measuring a second set of clamped currents.

In both cases, once a gradient estimate is available we update resistances by a (possibly constrained) step,
\begin{equation}
\br\leftarrow \mathrm{clip}\!\left(\br-\nabla_{\br}\mathcal{L},\,\br_{\min},\,\br_{\max}\right),
\end{equation}
where clipping is optional and enforces hardware bounds. When resistances are constrained to hardware bounds, the squared-error loss therefore attains a global minimizer; see Appendix~\ref{min_exists}.
 Empirically, we find that bounding $\br$ is often important for
stable two-phase training, whereas the $\Omega$-based updates are typically less sensitive.

\section{Example Tasks: Regression and Classification}
\label{sec:applications}

A passive resistive network implements a linear steady-state map from imposed edge sources to Ohmic drops,
\begin{equation}
\bv = -\Omega_{A/R}\,\bs,
\end{equation}
where $\Omega_{A/R}$ depends on the tunable resistances $R$. Choosing which edges are actuated and measured selects an
effective input--output block of this operator. With an input selector $P_i\in\{0,1\}^{E\times s}$ and output selector
$P_o\in\{0,1\}^{E\times t}$, encoding $\bx\in\mathbb{R}^{\mathrm{E_i}}$ as $\bs=\gamma P_i\bx$ yields the readout
\begin{equation}
\hat{\by} = f_R(\bx)
:= P_o^\top \bv
= -\gamma\,P_o^\top \Omega_{A/R} P_i\,\bx.
\label{eq:io_map_block}
\end{equation}
Training adjusts $R$ so that $f_R(\bx)$ matches a target mapping on a task distribution. Since $\Omega_{A/R}$ is a rank-$C$
projector ($C=E-N+1$), the realizable block satisfies the basic expressivity bound
$\mathrm{rank}(P_o^\top\Omega_{A/R}P_i)\le \min\{s,t,C\}$ (Appendix \ref{io_transformations}).



\subsection{Classification}
\label{sec:classification}

We next use the circuit as a binary classifier with a single readout edge. Given input $\bx$, the circuit produces an
edge-voltage pattern $\bv(\bx)$ and scalar score
\begin{equation}
f(\bx) := \mathbf{e}_o^\top \bv(\bx),
\end{equation}
with label $y\in\{-1,1\}$ predicted by $\mathrm{sign}(f(\bx))$. We train with hinge loss, see implementtion details in Appendix~\ref{app:classification_details}, so that only margin-violating
examples contribute updates; the corresponding circuit subgradient can be written directly in terms of steady-state
currents and $\Omega_{A/R}$ (derivation and implementation details in Appendix \ref{app:classification_details}).

We apply both gradient estimators to the Wisconsin breast cancer dataset after reducing the 30 features to 3 dimensions
via PCA. Figure~\ref{fig:cancer_classifier} summarizes training: both methods reach $\sim$90\% accuracy, while the
contrastive estimator shows more visible instability in the loss and learned boundary. Figure~\ref{fig:loss_landscapes}
visualizes the effective landscapes traversed by each method: the two trajectories agree early and then diverge as
resistances move away from the initial homogeneous regime.

\textbf{Performance under limited control}.
To model limited actuation, we freeze each edge resistance independently with probability $p_{\text{freeze}}$ at the
start of training, and apply updates only to the remaining edges. For each $p_{\text{freeze}}$ we repeat the experiment
over multiple random frozen subsets and report mean accuracy with error bars (Fig.~\ref{fig:pobs_cancer}). Specifically, we generate a masking vector $\bm{m}$ which is a Bernoulli random vector with success probability $p$. Then, for a given gradient, we instead apply the masked gradient $\bm{m} \odot \nabla_{\br}\mathcal{L}(\hat{\bm{y}},\bm{y})$.
Figure~\ref{fig:pobs_cancer} shows that, across random frozen subsets, the
projector-based estimator exhibits lower variance and more consistent performance improvement as access to edges increases, whereas
the two-phase estimator degrades more sharply under partial control. This implies that the control over all resistive states (e.g. updating the resistive value based on the gradient) is not strictly necessary.

\begin{figure}[ht!]
    \centering
    \includegraphics[width=0.99\linewidth]{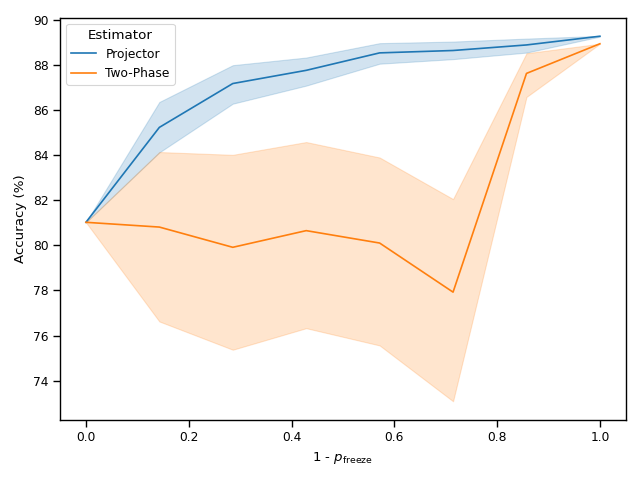}
    \caption{Accuracy on the Wisconsin dataset under varying degrees of circuit edge access for both gradient
    estimators. For each estimator and value of $p_{\text{freeze}}$, error bars and average accuracies are computed
    over 40 independent trials with randomly frozen edges (e.g. resistive values) for 1000 training steps; networks are always initialized
    from the everywhere homogeneous initial state.}
    \label{fig:pobs_cancer}
\end{figure}

\subsection{Noisy Regression with a disordered  network topology}
\label{sec:nanowire}
The discussion we have had so far has focused on a resistive network with a rectangular grid topology. To show that this method also works for random network topologies, we consider nanowire-inspired random resistive networks, which provide a topology that is both more irregular and more scalable. We use the nanowire network construction algorithm introduced in \cite{Zhu2021,Zhu2023} and based on the Monte Carlo deposition method of nanowires on a surface and checking for intersections. Specifically, we use the variant algorithm introduced in \cite{barrows2025ri} and described in App. \ref{app:random-nw-deposition}. 

We train on noisy linear
regression data
\begin{equation}
\by = M\bx + \varepsilon,
\end{equation}
with Gaussian noise $\varepsilon$  and $M\in\mathbb{R}^{2\times2}$ sampled uniformly entrywise.  
Under additive zero-mean noise, the analytical (projector) gradient remains unbiased while the two-phase current-squared estimator incurs both an $O(\beta)$ finite-nudge approximation error and a noise-induced statistical bias for any $\beta\neq 0$. For proof, as well as a detailed description of the task and setting, see Appendix~\ref{ref:StochGradient}.
\begin{figure}[ht!]
\centering

\subfloat[]{%
  \includegraphics[width=0.48\linewidth]{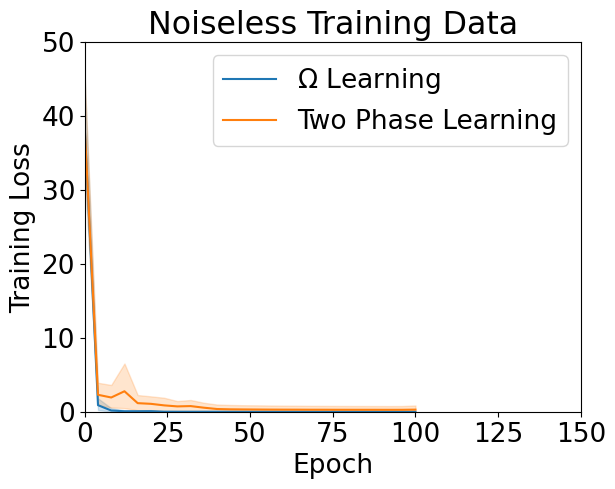}%
  \label{subfig:nanowire_losses_nl_0}%
}\hfill
\subfloat[]{%
  \includegraphics[width=0.48\linewidth]{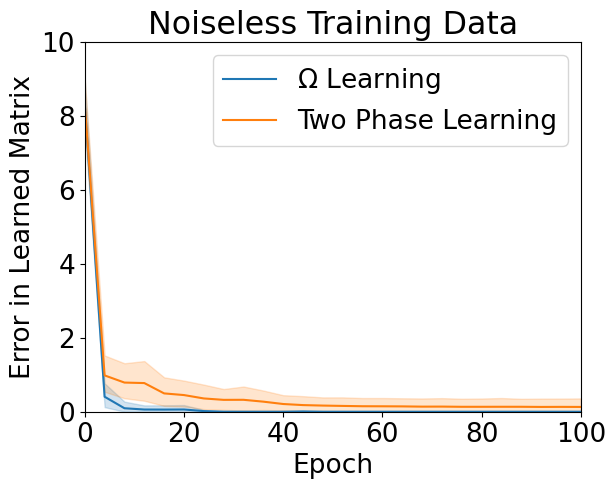}%
  \label{subfig:nanowire_matrix_errors_nl_0}%
}

\subfloat[]{%
  \includegraphics[width=0.48\linewidth]{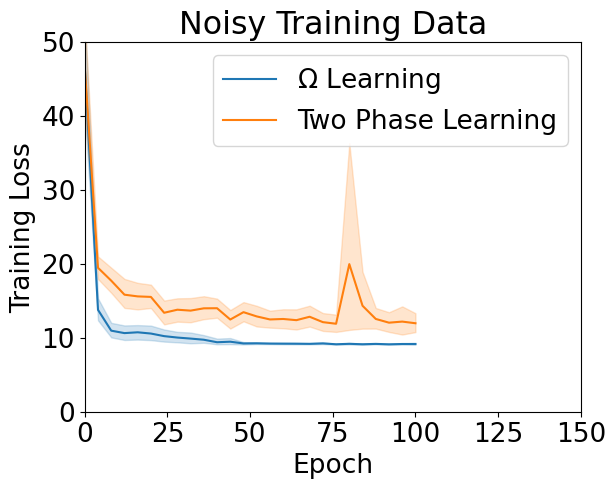}%
  \label{subfig:nanowire_losses_nl_3}%
}\hfill
\subfloat[]{%
  \includegraphics[width=0.48\linewidth]{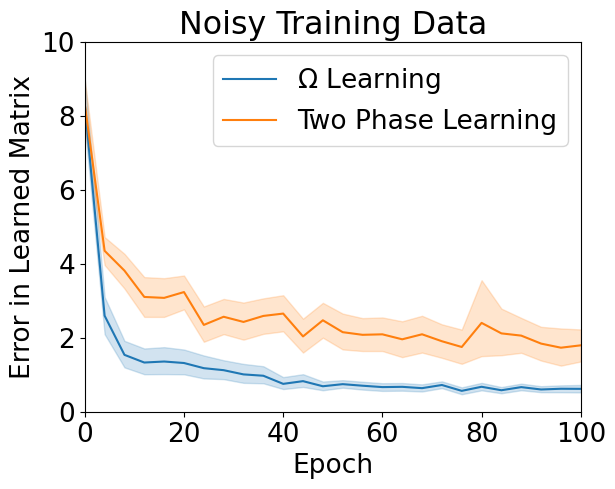}%
  \label{subfig:nanowire_matrix_errors_nl_3}%
}

\caption{Training 40 random resistive networks using both $\Omega$ learning and Two Phase Learning.
(a) and (c) Training losses for noiseless and noisy training data respectively.
(b) and (d) Error (Frobenius norm) between the learned input-output matrix and the true input-output matrix.}
\label{fig:nanowire_results}

\end{figure}

Figure
\ref{fig:nanowire_results} provides the loss averaged over 40  random nanowire-inspired networks for vanilla two-phase and  $\Omega$-based learning. We consider the regression task described above (both in the absence of and with noise in the training data). While the two methods are comparable in the noiseless regime, the results suggest that $\Omega$ learning provides an improvement in convergence speed and in obtaining a better fit. This matches the theoretical predictions presented in Appendix~\ref{ref:StochGradient}.


\section{Conclusion}
In this work, we reframed two-phase learning using linear-response theory, enabling EP- and CL-style updates to be compared to one another. We then specialized to passive linear
resistor networks; this viewpoint yields a compact operator description in terms of the response projector $\Omega_{A/R}$ and exposes what two-phase schemes approximate when $\beta$ is finite. We show that the $R$-weighted cycle-space projector $\Omega_{A/R}$ directly defines an exact gradient for learning in passive networks. Equilibrium propagation and coupled learning appear as approximate implementations of a similar projector-based learning rule.

Within this circuit setting we studied two training routes. The \emph{contrastive} rule is a genuine two-phase method,
producing a local difference-of-squares update from free and nudged steady states. The \emph{projector-based} rule
instead realizes the \emph{analytical} resistance gradient by physically implementing $\Omega_{A/R}$ and
$\Omega_{A/R}^\top$ with a small number of voltage- and current-mode experiments, avoiding finite-$\beta$ bias and
requiring only a single network (no replica).

We also showed that learnability is constrained by topology and by the choice of input/output edges: edge selection
amounts to tuning a submatrix of $\Omega_{A/R}$, and poor selections can limit expressivity. Experiments on regression
and binary classification indicate that projector-based training is typically more stable than two-phase learning,
while achieving comparable performance when the latter succeeds. Finally, our operator/thermodynamic perspective
suggests clear extensions to noise, dynamics, and nonlinear devices, and to co-design of graph structure, edge
selection, and learning rules. This will be the scope of future analysis.

\section*{Acknowledgements}
The author's work was conducted under the auspices of the National Nuclear Security Administration of the United States Department of Energy at Los Alamos National Laboratory (LANL) under Contract No. DE-AC52-06NA25396,   and was supported in part by the DOE Advanced 
Scientific Computing Research (ASCR) program under Award No.~DE-SCL0000118. J.L., A.D., and F.B.  also gratefully acknowledge support from the Center for Nonlinear Studies at LANL. F.C. is an employee of Planckian, but this work was initiated while a LANL employee.

\section*{Data availability}
The software package used to simulate this system can be found on GitHub\footnote{\url{https://github.com/jlin1212/differentiable-circuits}}, along with methods for dataset access and generation.

\section*{References}

\bibliography{bibliography}

\clearpage
\onecolumngrid
{\LARGE Appendix}
\subsection{Physical derivation of Equilibrium Propagation} \label{app:EPder}

Let $y = (y_1,\dots,y_n)^\top \in \mathbb{R}^n$ denote a set of coarse variables (generalized coordinates)
describing a physical system with the inputs clamped. We assume the system is in contact with a thermal
bath at (approximately) fixed temperature $T$; in this setting the appropriate thermodynamic potential for
relaxation is a \emph{free energy} (rather than the internal energy alone).

Accordingly, we postulate the existence of a smooth (non-equilibrium) free-energy function
\begin{equation}
  \mathcal F_\theta(\,\cdot\,;T): \mathbb{R}^n \to \mathbb{R},
\end{equation}
parametrized by $\theta \in \Theta \subset \mathbb{R}^p$. When it is useful to be explicit about its
thermodynamic meaning one may think of
\begin{equation}
  \mathcal F_\theta(y;T) \;=\; E_\theta(y) \;-\; T\,S(y),
  \label{eq:free-energy-def}
\end{equation}
where $E_\theta$ is an effective internal energy and $S$ an effective entropy associated with the coarse
description.\footnote{More generally, depending on which macroscopic constraints are held fixed, $F_\theta$
may represent the appropriate thermodynamic potential (e.g., Helmholtz or Gibbs free energy). For the present
derivation it suffices that $F_\theta$ is a Lyapunov-like potential whose gradient yields the thermodynamic forces.}
In the low-temperature limit, provided $S(y)$ remains finite (or grows slower than $1/T$),
\begin{equation}
  \lim_{T\to 0} \mathcal F_\theta(y;T) \;=\; E_\theta(y),
\end{equation}
so the formalism reduces to an energy-based description.

We define the corresponding \emph{thermodynamic forces} $X_i$ as components of the free-energy gradient:
\begin{equation}
  X_i(y) \;:=\; \frac{\partial {\mathcal F}_\theta(y;T)}{\partial y_i},
  \qquad i = 1,\dots,n,
\end{equation}
or, in vector form,
\begin{equation}
  X(y) \;:=\; \nabla_y {\mathcal F}_\theta(y;T).
\end{equation}
The associated \emph{fluxes} are the rates of change of the coarse variables,
\begin{equation}
  J_i \;:=\; \dot y_i,
  \qquad
  J(y) := \dot y.
\end{equation}

Near equilibrium, generalized forces $\{X_i\}$ and conjugate fluxes $\{J_i\}$ are related by Onsager
linear-response relations,
\begin{equation}
  J_i \;=\; -\sum_j L_{ij}X_j,
  \label{eq:onsager-scalar}
\end{equation}
with phenomenological coefficients $L_{ij}$ chosen so that relaxation produces nonnegative entropy production.
In compact notation, Eq.~\eqref{eq:onsager-scalar} reads
\begin{equation}
  J(y) \;=\; -\,L\,X(y),
  \label{eq:onsager-matrix}
\end{equation}
where $L \in \mathbb{R}^{n \times n}$ is the Onsager matrix. The matrix $L$ can be decomposed into symmetric and
antisymmetric parts,
\begin{equation}
  L \;=\; \Gamma - A,
  \qquad
  \Gamma^\top = \Gamma \succeq 0,\;\; A^\top = -A,
\end{equation}
where $\Gamma$ encodes dissipative couplings and $A$ encodes reversible couplings. In terms of $y(t)$, the
dynamics reads
\begin{equation}
  \dot y
  \;=\; -\,\Gamma\,\nabla_y {\mathcal F}_\theta(y;T) \;+\; A\,\nabla_y {\mathcal F}_\theta(y;T).
  \label{eq:y-dynamics-Onsager}
\end{equation}

The time derivative of the free energy along a trajectory $t \mapsto y(t)$ is
\begin{align}
  \frac{d}{dt} {\mathcal F}_\theta(y(t);T)
  &= \nabla_y {\mathcal F}_\theta(y;T)^\top \dot y \nonumber \\
  &= X(y)^\top \big( -\,\Gamma\,X(y) + A\,X(y) \big) \nonumber \\
  &= -\,X(y)^\top \Gamma X(y),
  \label{eq:Fdot}
\end{align}
since $X^\top A X = 0$ for any antisymmetric $A$. Because $\Gamma$ is positive semidefinite,
\begin{equation}
  \frac{d}{dt} {\mathcal F}_\theta(y(t);T) \;\le\; 0,
\end{equation}
with equality only at points where $X(y) = 0$ on the dissipative subspace. Thus $\mathcal F_\theta$ is a Lyapunov function for
\eqref{eq:y-dynamics-Onsager}, and the dynamics relaxes toward stable critical points of $F_\theta(\cdot;T)$ (and, in the
limit $T\to 0$, of $U_\theta$) compatible with the constraints encoded by $A$. In the following, we assume that $T$ is sufficiently small, and that we can approximate $\mathcal F_\theta$ with the energy $E_\theta$. Otherwise, the formalism follows naturally replacing $E_\theta$ with $\mathcal F_\theta$.

In many situations we may neglect reversible couplings, or treat them separately and retain only the dissipative part.
This yields the gradient-flow limit
\begin{equation}
  \dot y \;=\; -\,\Gamma\,\nabla_y E_\theta(y),
  \qquad \Gamma \succ 0,
  \label{eq:gradflow-flux}
\end{equation}
which we adopt as the baseline deterministic dynamics for the derivation of EP.

Equilibrium Propagation introduces learning by applying a small static \emph{training field} that biases the system
toward desired outputs. In the thermodynamic language above, this corresponds to modifying the energy by an additional
term that depends on a small control parameter $\beta > 0$ (note that $\beta$ below is not $1/k_B T$ but a generic parameter):
\begin{equation}
  F(\beta, \theta, y)
  \;:=\;
  E_\theta(y) + n(\beta,\theta,y),
  \label{eq:nudged-energy-flux}
\end{equation}
where the nudge $n(\beta,\theta,y)$ satisfies
\begin{equation}
  n(0,\theta,y) = 0,
  \qquad
  n(\beta,\theta,y) = O(\beta) \quad \text{as } \beta \to 0.
\end{equation}
The corresponding forces are
\begin{equation}
  X^\beta(y) \;=\; \nabla_y F(\beta, \theta, y)
  \;=\; \nabla_y E_\theta(y) + \nabla_y n(\beta,\theta,y),
\end{equation}
and the dissipative dynamics becomes
\begin{equation}
  \dot y
  \;=\;
  -\,\Gamma\,\nabla_y F(\beta, \theta, y).
\end{equation}

Equilibrium Propagation is recovered as the special case
\begin{equation}
  n(\beta,\theta,y) = \beta\,C(\theta,y),
\end{equation}
where $C(\theta,y)$ is a cost function that measures the disagreement between system outputs and targets. In that case,
the nudging corresponds to a linear coupling between the physical degrees of freedom and the cost.

For fixed $(\beta,\theta)$, the dynamics relaxes to a (locally) stable equilibrium $y_\theta^\beta$ satisfying
\begin{equation}
  \nabla_y F(\beta, \theta,y_\theta^\beta) = 0.
\end{equation}
The \emph{free phase} corresponds to $\beta = 0$ and equilibrium $y_\theta^0$ determined by the unperturbed energy
$E_\theta$. The \emph{nudged phase} corresponds to a small but nonzero $\beta > 0$ and equilibrium $y_\theta^\beta$ of
the perturbed energy $E_\theta^\beta$.

The key observation behind EP is that, in the linear-response regime $\beta \to 0$, the difference between local
observables evaluated at $y_\theta^\beta$ and $y_\theta^0$ encodes the gradient of a training objective with respect to
$\theta$. This leads to a two-phase learning rule in which parameter updates are computed from local measurements in the
free and nudged phases.

To make this precise, it is useful to distinguish between \emph{partial} derivatives with respect to explicit arguments
and \emph{total} derivatives that also account for the dependence of the equilibrium state on those arguments.

Let $h(\beta,\theta,y^\beta)$ be any smooth function, where $y^\beta$ denotes the (locally) stable equilibrium at
a given pair $(\beta,\theta)$. We define the partial derivative with respect to the explicit argument $\beta$ as
\begin{equation}
  \partial_\beta h(\beta,\theta,y^\beta)
  :=
  \lim_{\alpha \to 0} \frac{1}{\alpha}
  \Big( h(\beta + \alpha,\theta,y^\beta) - h(\beta,\theta,y^\beta) \Big),
\end{equation}
i.e., differentiation with respect to $\beta$ while holding $y^\beta$ fixed. In contrast, the total derivative of $h$
with respect to $\beta$ is
\begin{equation}
  \frac{d}{d\beta} h(\beta,\theta,y^\beta)
  :=
  \lim_{\alpha \to 0} \frac{1}{\alpha}
  \Big( h(\beta + \alpha,\theta,y^{\beta+\alpha}) - h(\beta,\theta,y^\beta) \Big),
\end{equation}
which includes both the explicit dependence on $\beta$ and the implicit dependence via the equilibrium
$y_\theta^\beta$. We use the analogous notation $\partial_\theta$ and $\dfrac{d}{d\theta}$ for differentiation with
respect to the parameters $\theta$.

Given the nudged energy in \eqref{eq:nudged-energy-flux}, we consider dynamics with the same Onsager operator
$\Gamma$:
\begin{equation}
  \dot y \;=\; -\,\Gamma\,\nabla_y F(\beta, \theta, y).
\end{equation}
The corresponding energy dissipation is
\begin{eqnarray}
  \frac{d}{dt} F(\beta, \theta, y(t))
  &=& \nabla_y F(\beta, \theta, y)^\top \dot y
  \nonumber \\
  &=& \big( \nabla_y E_\theta(y) + \nabla_y n(\beta,\theta,y) \big)^\top
      \big( -\,\Gamma \big( \nabla_y E_\theta(y) + \nabla_y n(\beta,\theta,y) \big) \big)
  \nonumber \\
  &=& -\,\big\| \Gamma^{1/2} \big( \nabla_y E_\theta(y) + \nabla_y n(\beta,\theta,y) \big) \big\|^2
      \;\le\; 0,
  \label{eq:nudged-diss}
\end{eqnarray}
using that $\Gamma \succeq 0$. Thus, for fixed $\beta$, the nudged energy $E_\theta^\beta$ is again nonincreasing along
trajectories and plays the role of a Lyapunov function for the nudged dynamics.

In this formulation, the training field enters by modifying the energy itself. An alternative would be to keep
$E_\theta$ unchanged and add $\nabla_y n(\beta,\theta,y)$ as an extra nonconservative force on top of
$\nabla_y E_\theta(y)$. In that case the dynamics would, in general, no longer be of pure gradient-flow form and one
would lose the simple monotonicity property in \eqref{eq:nudged-diss}, which is what we will rely on when deriving
linear-response identities for Equilibrium Propagation.

We now characterize how the equilibrium $y_\theta^\beta$ changes with the nudging strength $\beta$.
For each fixed $(\beta,\theta)$, the equilibrium satisfies the stationarity condition
\begin{equation}
  \nabla_y F(\beta, \theta, y_\theta^\beta)
  \;=\;
  \nabla_y E_\theta(y_\theta^\beta) + \nabla_y n(\beta,\theta,y_\theta^\beta)
  \;=\; 0.
  \label{eq:stat-nudged}
\end{equation}
Let $y^0 := y_\theta^0$ denote the free-phase equilibrium at $\beta = 0$, so that
\begin{equation}
  \nabla_y E_\theta(y^0) = 0.
\end{equation}
We assume that $y^0$ is a (locally) stable equilibrium and define the Hessian
\begin{equation}
  H \;:=\; \nabla_y^2 E_\theta(y^0).
\end{equation}

For small $\beta$, we expand \eqref{eq:stat-nudged} around $(\beta,y) = (0,y^0)$. Writing
$y^\beta := y_\theta^\beta$ and using $\nabla_y E_\theta(y^0)=0$ gives, to first order in $\beta$ and
$y^\beta - y^0$,
\begin{eqnarray}
  0
  &=&
  \nabla_y E_\theta(y^\beta) + \nabla_y n(\beta,\theta,y^\beta)
  \nonumber \\
  &\approx&
  \nabla_y E_\theta(y^0)
  + H\,(y^\beta - y^0)
  + \nabla_y n(\beta,\theta,y^0)
  \nonumber \\
  &=&
  H\,(y^\beta - y^0) + \nabla_y n(\beta,\theta,y^0).
\end{eqnarray}
Thus
\begin{equation}
  y^\beta - y^0
  \;=\;
  -\,H^{-1}\,\nabla_y n(\beta,\theta,y^0) + O(\beta^2).
\end{equation}
Using that $n(\beta,\theta,y) = O(\beta)$ as $\beta \to 0$, we expand
\begin{equation}
  \nabla_y n(\beta,\theta,y^0)
  =
  \beta\,\nabla_y \partial_\beta n(0,\theta,y^0) + O(\beta^2),
\end{equation}
which yields
\begin{equation}
  y^\beta - y^0
  \;=\;
  -\,\beta\,H^{-1}\,\nabla_y \partial_\beta n(0,\theta,y^0) + O(\beta^2).
  \label{eq:IFT-state}
\end{equation}
Equivalently, the \emph{static susceptibility} of the equilibrium with respect to the nudging strength is
\begin{equation}
  \frac{d y^\beta}{d\beta}\Big|_{\beta=0}
  \;=\;
  -\,H^{-1}\,\nabla_y \partial_\beta n(0,\theta,y^0).
  \label{eq:susceptibility}
\end{equation}
(If $H$ is only semidefinite, $H^{-1}$ may be interpreted as the inverse restricted to the stable subspace, or as the
Moore--Penrose pseudoinverse.)

Note that \eqref{eq:susceptibility} depends only on the energy landscape $E_\theta$ and the form of the nudge
$n(\beta,\theta,y)$, but not on the specific choice of the Onsager operator $\Gamma$ in the dynamics: $\Gamma$ affects
how the system relaxes to equilibrium, but not the location of the equilibrium itself, when $\Gamma \succ 0$.


We now derive the equilibrium propagation identity in this physical setting. Define
\begin{equation}
  \ell(y) \;:=\; \partial_\beta n(0,\theta,y),
\end{equation}
and take the training objective to be the cost evaluated at the free-phase equilibrium:
\begin{equation}
  J(\theta) \;:=\; \ell(y^0(\theta)),
  \label{eq:J-def}
\end{equation}
where $y^0(\theta)$ is a (locally) stable equilibrium of $E_\theta$. In standard Equilibrium Propagation with
\begin{equation}
  n(\beta,\theta,y) = \beta\,C(\theta,y),
\end{equation}
we have $\ell(y) = C(\theta,y)$ and therefore $J(\theta) = C(\theta,y^0)$.

We first express $\nabla_\theta J(\theta)$ in terms of the response of the free equilibrium $y^0$ to a change in
$\theta$. Differentiating the free equilibrium condition
\begin{equation}
  \nabla_y E_\theta(y^0) = 0
\end{equation}
with respect to $\theta$ and using the implicit function theorem gives
\begin{equation}
  \frac{\partial y^0}{\partial \theta}
  \;=\;
  -\,H^{-1}\,\nabla_y \partial_\theta E_\theta(y^0),
  \qquad
  H \;:=\; \nabla_y^2 E_\theta(y^0),
  \label{eq:IFT-param}
\end{equation}
provided $H$ is invertible on the stable subspace.

Using $J(\theta) = \ell(y^0(\theta))$ and the chain rule, we obtain
\begin{eqnarray}
  \nabla_\theta J(\theta)
  &=& \left( \nabla_y \ell(y^0) \right)^\top \frac{\partial y^0}{\partial \theta}
  \nonumber \\
  &=& \left( \nabla_y \ell(y^0) \right)^\top
      \Big( -\,H^{-1}\,\nabla_y \partial_\theta E_\theta(y^0) \Big).
  \label{eq:grad-J}
\end{eqnarray}

On the other hand, we can relate $\nabla_\theta J(\theta)$ to the difference of energy derivatives between the nudged
and free equilibria. Consider the quantity
\begin{equation}
  \partial_\theta F(\beta, \theta, y^\beta) - \partial_\theta E_\theta^0(y^0),
\end{equation}
where $E_\theta^\beta = E_\theta + n(\beta,\theta,\cdot)$ and $y^\beta$ is the nudged equilibrium. We decompose it as
\begin{eqnarray}
  \partial_\theta F(\beta, \theta, y^\beta) - \partial_\theta E_\theta^0(y^0)
  &=&
  \big( \partial_\theta F(\beta, \theta, y^\beta) - \partial_\theta E_\theta^0(y^\beta) \big)
   + \big( \partial_\theta E_\theta^0(y^\beta) - \partial_\theta E_\theta^0(y^0) \big).
\end{eqnarray}
The first bracket is just the explicit $\theta$–dependence introduced by the nudge:
\begin{equation}
  \partial_\theta F(\beta, \theta, y^\beta) - \partial_\theta E_\theta^0(y^\beta)
  \;=\; \partial_\theta n(\beta,\theta,y^\beta).
\end{equation}
If $n$ has no explicit $\theta$–dependence (for example $n(\beta,\theta,y) = \beta\,\ell(y)$), this term vanishes.
When $\partial_\theta n \neq 0$ but is $O(\beta)$, it contributes an $O(\beta)$ correction to the EP identity that can
be handled separately. For clarity, we focus on the common case where $n$ depends on $\theta$ only through $y$ and
hence
\begin{equation}
  \partial_\theta n(\beta,\theta,y^\beta) = 0.
\end{equation}

We then expand the second term to first order in $y^\beta - y^0$:
\begin{eqnarray}
  \partial_\theta E_\theta^0(y^\beta) - \partial_\theta E_\theta^0(y^0)
  &\approx&
  \nabla_y \partial_\theta E_\theta(y^0)^\top (y^\beta - y^0)
  + O(\|y^\beta - y^0\|).
\end{eqnarray}
Using the linear response of the equilibrium state with respect to $\beta$ in \eqref{eq:IFT-state}, and the definition
$\ell(y) = \partial_\beta n(0,\theta,y)$, we have
\begin{equation}
  y^\beta - y^0
  \;=\;
  -\,\beta\,H^{-1}\,\nabla_y \ell(y^0) + O(\beta^2).
\end{equation}
Hence
\begin{eqnarray}
  \partial_\theta F(\beta, \theta, y^\beta) - \partial_\theta E_\theta^0(y^0)
  &\approx&
  \nabla_y \partial_\theta E_\theta(y^0)^\top \big( -\,\beta\,H^{-1}\,\nabla_y \ell(y^0) \big)
  + O(\beta^2)
  \nonumber \\
  &=&
  -\,\beta\,\nabla_y \partial_\theta E_\theta(y^0)^\top H^{-1}\,\nabla_y \ell(y^0)
  + O(\beta^2).
\end{eqnarray}
Dividing by $\beta$ and taking the limit $\beta \to 0$ yields
\begin{equation}
  \lim_{\beta \to 0} \frac{1}{\beta}
  \Big( \partial_\theta F(\beta, \theta, y^\beta) - \partial_\theta E_\theta^0(y^0) \Big)
  \;=\;
  -\,\nabla_y \partial_\theta E_\theta(y^0)^\top H^{-1}\,\nabla_y \ell(y^0).
\end{equation}
Comparing with \eqref{eq:grad-J} shows that
\begin{equation}
  \nabla_\theta J(\theta)
  \;=\;
  \lim_{\beta \to 0} \frac{1}{\beta}
  \Big[
    \partial_\theta F(\beta, \theta, y^\beta) - \partial_\theta E_\theta^0(y^0)
  \Big],
  \label{eq:ep-identity}
\end{equation}
which is the equilibrium propagation (EP) identity \cite{ScellierBengio2017EP}. It expresses the gradient of the training
objective with respect to parameters as the linear-response limit of a difference of energy derivatives evaluated at the
nudged and free equilibria.

In practice, simulations of energy-based models often approximate the continuous-time dynamics
\begin{equation}
  \dot y = -\,\Gamma\,\nabla_y F(\beta, \theta, y)
\end{equation}
by explicit Euler updates with step size $\eta > 0$. In this case, the discrete free and nudged phases are described by
\begin{align}
  \text{Free:}\qquad
  &y^0_{k+1} \;=\; y^0_k - \eta\,\Gamma\,\nabla_y E_\theta(y^0_k),
  \label{eq:euler-free}
  \\
  \text{Nudged:}\qquad
  &y^\beta_{k+1} \;=\; y^\beta_k - \eta\,\Gamma\Big(\nabla_y E_\theta(y^\beta_k)
     + \nabla_y n(\beta,\theta,y^\beta_k)\Big).
  \label{eq:euler-nudged}
\end{align}
Under standard smoothness and stability assumptions on $E_\theta$ and a sufficiently small step size $\eta$, these
iterations converge to the continuous equilibria $y^0$ and $y^\beta$ respectively, and the EP identity
\eqref{eq:ep-identity} can be approximated from finite differences measured between the two phases.

\subsection{Generalized Equilibrium Propagation}
\label{app:gep}

In this section, let $y = (y_1,\dots,y_n)^\top \in \mathbb{R}^n$ denote a set of coarse variables (generalized coordinates) describing
an input-clamped physical system. Here we wish to show in which sense Equilibrium Propagation (EP) and Coupled Learning (CL) are both two-phase schemes, one compares a free steady
state to a nearby, weakly biased (nudged) steady state and extracts a parameter update from their difference. The key
difference is the order of the perturbation with respect to the nudging strength. To treat both cases within a
single perturbative framework, we introduce \emph{Generalized Equilibrium Propagation}. 
We define a $\beta$-dependent (nudged) free energy
\begin{equation}
  F(\beta, \theta, y; \bz):= {\mathcal F}_\theta(y; \bz) + n(\beta,\theta,y; \bz),
  \qquad \beta>0
  \label{eq:nudged-energy}.
\end{equation}
such that $n(0,\theta,y; \bz)=0$ for all $\theta, y$). In this section, we will generally omit the dependence on the clamped input $\bz$ for brevity. We will also consider the low temperature limit, e.g. $|T S|\ll |E|$, but everything follows below by replacing $E_\theta$ with $\mathcal F_\theta$ in the finite temperature case.

Let $y^0(\theta)$ denote a (free) equilibrium of $E_\theta$, i.e.\ $\nabla_y E_\theta(y^0)=0$, and let $y^\beta(\theta)$ denote a (nudged) equilibrium of $F(\beta,\theta,\cdot)$, i.e.\ $\nabla_y F(\beta,\theta,y^\beta)=0$ for each $\beta$. We assume that $n(\beta, \theta, y^{0}) = O(\beta^k)$ for some constant $k \geq 1$.

Note that as long as $y^{\beta}$ is a differentiable function of $\beta$, then $\|y^{\beta} - y^0\| = O(\beta^k)$. To see this, note that Equation \eqref{eq:nudged-energy} implies that
\begin{align}
  \nabla_{\y} F(\beta, \theta, y^{\beta}) &= \nabla_{\y} E_\theta(y^{\beta})\ + \nabla_{\y} n(\beta, \theta, y^{\beta}).
\end{align}
Assume that the Hessian $H_\theta := \nabla_y^2 E_\theta(y^0)$ is invertible (nondegenerate equilibrium). Since $y^{\beta}$ is a differentiable function of $\beta$, 
\begin{align}
  \nabla_{\y} F(\beta, \theta, y^{\beta})
  &= \nabla_{\y} E_\theta(y^{0})
     + \nabla_{\y}^2 E_\theta(y^{0}) (\y^\beta - \y^0)
     + O(\| \y^\beta - \y^0 \|^2) + \nabla_{\y} n(\beta, \theta, y^{0}) + \nabla_{\y}^2 n(\beta, \theta, y^{0}) (\y^\beta - \y^0) + O(\| \y^\beta - \y^0 \|^2) 
\end{align}
By definition, $\nabla_{\y} E_\theta(y^{0}) = 0$ and $\nabla_{\y} F(\beta, \theta, y^{\beta}) = 0$. Also, since $n(\beta, \theta, y^{0}) = O(\beta^k)$, $\nabla_{\y} n(\beta, \theta, y^{0}) = O(\beta^k)$. Thus
\begin{align}
  0&= \left (H_\theta + \nabla_{\y}^2 n(\beta, \theta, y^{0}) \right )(\y^\beta - \y^0)
     + O(\beta^k) + O(\| \y^\beta - \y^0 \|^2) 
\end{align}

Note that $H_\theta$ is invertible and $\nabla_{\y}^2 n(\beta, \theta, y^{0})$ is a differentiable function of $\beta$. Thus, for sufficiently small $\beta$, $H_\theta + \nabla_{\y}^2 n(\beta, \theta, y^{0})$ is invertible and $(H_\theta + \nabla_{\y}^2 n(\beta, \theta, y^{0}))^{-1} = O(1)$.

Left-multiplying both sides by $(H_\theta + \nabla_{\y}^2 n(\beta, \theta, y^{0}))^{-1}$ gives us
\begin{align}
  0 &= \y^\beta - \y^0 + O(\| \y^\beta - \y^0 \|^2) + O(\beta^k) .
\end{align}
For small $\beta$, the $O(\| \y^\beta - \y^0 \|^2)$ term  is negligible compared to the other terms, so we conclude that $\|y^\beta - y^0\| = O(\beta^k)$.

For small $\beta$, we then have the Taylor expansion
\begin{align}
    n(\beta, \theta, y^{0}) &= 
    n(0, \theta, y^{0}) + \sum_{j = 1}^{\infty}
    {\beta^j \over j!} {\partial^j n \over \partial \beta^j} (0, \theta, y^{0})
    \nonumber \\
    &= 
    {\beta^k \over k!} {\partial^k n \over \partial \beta^k} (0, \theta, y^0) + O(\beta^{k+1})
    \nonumber \\
    &= \beta^k h(\theta, y^0) + O(\beta^{k+1})
\end{align}
where here we define 
\begin{align}
    h(\theta, y) :={1 \over k!} {\partial^k n \over \partial \beta^k} (0, \theta, y).
\end{align}

We also assume that $F(\beta, \theta, y^\beta) - F(\beta, \theta, y) = O(\| \y^{\beta} - \y\|^m)$ for some constant $m > 1$. Then
\begin{align}
    F(\beta, \theta, y^{\beta}) - F(\beta, \theta, y^0)
    &=
    O(\| \y^{\beta} - \y^0\|^m)
    \nonumber \\
    &=
    O(\beta^{mk}), 
\end{align}
where the last line follows from the fact that $\|y^{\beta} - y^0\| = O(\beta^k)$. 

Next, note that by the chain rule,

\begin{align}
    {d F \over d \theta}(\beta, \theta, y^{\beta}) -
    {d F \over d \theta}(0, \theta, y^{0})
    &=
    {\partial F \over \partial \theta}(\beta, \theta, y^{\beta})
    +
    {\partial F \over \partial y}(\beta, \theta, y^{\beta}) \cdot {\partial y^{\beta} \over \partial \theta}(\theta)
    -
    \left ( {\partial F \over \partial \theta}(0, \theta, y^{0})
    +
    {\partial F \over \partial y}(0, \theta, y^{0}) \cdot {\partial y^{0} \over \partial \theta}(\theta)
    \right )
     \nonumber \\
    &=
    {\partial F \over \partial \theta}(\beta, \theta, y^{\beta})
    - {\partial F \over \partial \theta}(0, \theta, y^{0})
\end{align}

Then
\begin{align}
    {\partial F \over \partial \theta}(\beta, \theta, y^{\beta}) -
    {\partial F \over \partial \theta}(0, \theta, y^{0})
    &=
    {d F \over d \theta}(\beta, \theta, y^{\beta}) -
    {d F \over d \theta}(0, \theta, y^{0})
    \nonumber \\
    &= 
    {d \over d \theta} \left ( F(\beta, \theta, y^{\beta}) - F (\beta, \theta, y^{0}) 
    + F (\beta, \theta, y^{0})  -
    F(0, \theta, y^{0}) \right )
    \nonumber \\
    &=
    {d \over d \theta} \left (
    O(\beta^{mk}) + n(\beta, \theta, y^{0})\right )
    \nonumber \\
    &=
    {d \over d \theta} \left ( \beta^k h(\theta, y^{0}) + O(\beta^{k+1})+ O(\beta^{mk}) \right ). 
\end{align}

Since $k \geq 1$ and $m > 1$,  we have $k < mk$. Thus
\begin{align}
    \lim_{\beta^k \to 0} {1 \over \beta^k} \left ( {\partial F \over \partial \theta}(\beta, \theta, y^{\beta}) -
    {\partial F \over \partial \theta}(0, \theta, y^{0}) \right )
    &=
    \lim_{\beta^k \to 0} {1 \over \beta^k}  {d \over d \theta} \left (\beta^k h(\theta, y^{0}) + O(\beta^{k+1})+ O(\beta^{mk})\right )
    \nonumber \\
    &=
    \lim_{\beta^k \to 0} {d \over d \theta} \left (h(\theta, y^0) + O(\beta)
    + O(\beta^{mk - k}) \right )
    \nonumber \\
    &=
    {d h\over d \theta} (\theta, y^0) .
\end{align}

If we define the objective function as
\begin{align}
    J(\theta) = h(\theta, y^0),
\end{align}
we have
\begin{align}
    {\partial J\over \partial \theta} \left(\theta \right)
    =
    \lim_{\beta^k \to 0} {1 \over \beta^k} \left ( {\partial F \over \partial \theta}(\beta, \theta, y^{\beta}) -
    {\partial F \over \partial \theta}(0, \theta, y^{0}) \right )
    \label{eq: GEP}
\end{align}

In Equilibrium Propagation, $n(\beta, \theta, y^{0}) = \beta C(\theta, y_{0})$, so $k = 1$ and $h(\theta, y^0) = C(\theta, y_{0})$. Plugging this into the above yields the standard equilibrium propagation identity.

\subsection{EP and CL as two-phase nudge}
\label{app:epcl}

In this section (as in the previous one), we let $y = (y_1,\dots,y_n)^\top \in \mathbb{R}^n$ denote a set of coarse variables (generalized coordinates) describing
an input-clamped physical system.

\textbf{Equilibrium Propagation as a Special Case of Generalized Equilibrium Propagation}.

In Equilibrium Propagation, the nudge is linear in $\beta$,
\begin{align}
  n(\beta, \theta, y) &= \beta C(\theta, y),
\end{align}
so $k=1$ and $h(\theta,y^0)=C(\theta,y^0)$.

Equation \eqref{eq: GEP} then reduces to the usual EP identity with training objective $J(\theta) = C(\theta,y^0(\theta))$.

Also,

\begin{align}
     F(\beta, \theta, y) - F(\beta, \theta, y^{\beta})
     &=
     F(\beta, \theta, y^{\beta}) +(y - y^{\beta}) \nabla_{y} F(\beta, \theta, y^{\beta}) + O(\|y - y^{\beta}\|^2) -  F(\beta, \theta, y^{\beta})
     \\
     &= O(\|y - y^{\beta}\|^2),
    \nonumber
\end{align}

so $m = 2$.

\textbf{Coupled Learning as a Special Case of Generalized Equilibrium Propagation}.

In the case of  Coupled Learning: Assuming that the correct output is $v_T$, we have the nudged energy
\begin{align}
    F(\beta, \theta, y) = F(0, \theta, y + \epsilon_{CL} )
\end{align}
where $\epsilon_{CL}  =\beta P_o^{\top}(v_T - P_o y)$ is an error vector and $P_o$ is a linear operator that selects the output states of the coarse variable $y$.

Then
\begin{align}
    n(\beta, \theta, y^0) &= F(\beta, \theta, y^0) - F(0, \theta, y^0)
    \nonumber \\
    &= F(0, \theta, y^0 + \beta P_o^{\top}(v_T - P_o y^0) ) - F(0, \theta, y^0)
    \nonumber  \\
    &= \beta (v_T - P_o y^0)^{\top} P_o  \nabla_y F(0, \theta, y^0) + \frac{\beta^2}{2} (v_T - P_o y^0)^{\top} P_o \nabla_y^2 F(0, \theta, y^0) P_o^{\top} (v_T - P_o y^0) + O(\beta^3)
    \nonumber  \\
    &= \frac{\beta^2}{2} (v_T - P_o y^0)^{\top} P_o \nabla_y^2 F(0, \theta, y^0) P_o^{\top} (v_T - P_o y^0) + O(\beta^3)
\end{align}
since $\nabla_y F(0,\theta,y^0) = 0$ at the free equilibrium $y^0$. Thus $n(\beta, \theta, y)$ is second order in $\beta$, i.e. $k = 2$ in the GEP framework. We also have $m = 2$ by the same reasoning as in the case of Equilibrium Propagation.

\subsection{Physical biasing schemes and projectors}
\label{app:projectors}

We now make explicit the \emph{linear-operator} viewpoint that underlies our biasing and gradient-estimation schemes.
Consider a connected directed graph $\mathcal{G}$ with $N$ nodes and $\mrE$ edges, together with a fixed orientation
$\mathcal{O}$ (i.e., an arbitrary but fixed choice of direction for each edge). Each edge $e$ consists of an ideal
voltage source $\bs_e$ in series with a resistor of resistance $r_e>0$. We stack edgewise
quantities into vectors: sources $\bs\in\mathbb{R}^{\mrE}$, currents $\bi\in\mathbb{R}^{\mrE}$, and Ohmic drops
$\bv:=R\bi\in\mathbb{R}^{\mrE}$, where
\begin{equation}
R=\diag(r_e)\succ 0.
\end{equation}
Let $A\in\mathbb{R}^{C\times \mrE}$ be a cycle (loop) matrix spanning the $C=\mrE-N+1$ fundamental cycles of the graph. The
cycle constraints encode Kirchhoff's Voltage Law (KVL): the signed sum of voltage drops around each fundamental cycle
must be zero. With an edge source $\bs$ and Ohmic drops $R\bi$, this gives
\begin{equation}
A(\bs+\bv)=\bf 0.
\label{eq:kvl_sources_Ohmic}
\end{equation}


The total power dissipated in the resistors, $\bv^\top R^{-1} \bv$ is minimized subject to the above constraint.

We can solve this constrained optimization problem using the method of Lagrange multipliers. Define the Lagrangian:
\begin{equation}
L(\bv, \mathbf{\lambda}; \bs) = \bv^\top R^{-1} \bv + \mathbf{\lambda}^\top A (\bs + \bv)
\end{equation}
where $\mathbf{\lambda} \in \mathbb{R}^C$ is the vector of Lagrange multipliers and $A$ is the cycle matrix. 

Now we have
\begin{equation}
\nabla_{\bv} L(\bv^*, \lambda^*; \bs) = 2 R^{-1} \bv^* + A^{\top} \lambda^* = 0
\end{equation}
Thus $\bv^* = -{1 \over 2} R A^{\top} \lambda^*$.
Similarly,
\begin{equation}
\nabla_{\lambda} L(\bv^*, \lambda^*; \bs) = A (\bs + \bv^*) = 0,
\end{equation}
so
\begin{align}
A (\bs -{1 \over 2} R A^{\top} \lambda^*) &= 0
\end{align}
and thus $\lambda^* = 2 (A R A^{\top})^{-1} A \bs$. Substituting this back into our expression for $\bv^*$ gives us $\bv^* = - R A^{\top}(A R A^{\top})^{-1} A \bs$.

This motivates defining the \emph{cycle projector}
\begin{equation}
\Omega_{A/R}:=RA^\top(ARA^\top)^{-1}A,
\label{eq:omega_def_repeat}
\end{equation}
so that the circuit implements the linear map
\begin{equation}
\bv = -\Omega_{A/R}\,\bs,
\qquad \text{and} \qquad
\bi = -R^{-1}\Omega_{A/R}\,\bs.
\label{eq:operator_map_sv}
\end{equation}
A key structural fact is that $\Omega_{A/R}$ is an \emph{oblique projector}:
\begin{equation}
\Omega_{A/R}^2=\Omega_{A/R}.
\end{equation}
Intuitively, $\Omega_{A/R}$ extracts the component of an imposed edge-source pattern that is compatible with the
graph's cycle constraints and converts it into the steady-state distribution of Ohmic drops. Equivalently, each column
of $-\Omega_{A/R}$ is the network response (Ohmic drops on \emph{all} edges) to a unit source applied on a single edge.

Equations \eqref{eq:operator_map_sv} 
provide the bridge between the physical two-phase
picture and the projector-based schemes used later: voltage-mode experiments implement $\bs\mapsto\bv$ (left
multiplication by $-\Omega_{A/R}$), while reciprocal/current-mode manipulations can be used to access the adjoint
$\Omega_{A/R}^\top$ in the projector-based gradient estimator developed in the next subsections.

This operator viewpoint is useful for two reasons. First, it makes inference explicit: a \emph{voltage-mode} experiment
that imposes $\bs$ and measures $\bv$ is exactly a multiplication by $-\Omega_{A/R}$. Second, it clarifies what must be
implemented physically for learning: our contrastive (two-phase) rule estimates gradients from differences between two
such voltage-mode experiments, while our projector-based estimator additionally requires access to the adjoint action
$\Omega_{A/R}^\top$ via reciprocal/current-mode manipulations.

A local minima of Equation \eqref{eq:ProjectorAnalyticalGradient_MainText} satisfies 
\begin{align}
    ((\mathbf{v} - \mathbf{v}_T)^{\top}(I-\Omega_{A/R})\,\text{diag}(\mathbf{i}))^{\top} &= \bf 0 \nonumber
    \\
    \text{diag}(\mathbf{i})(I-\Omega_{A/R}^{\top})(\mathbf{v} - \mathbf{v}_T) &= \bf 0 .
\end{align}
Assuming currents are nonzero on all edges (and thus $\text{diag}(\mathbf{i})$ is full rank), this means that
\begin{align}
    (I-\Omega_{A/R}^{\top})(\mathbf{v} - \mathbf{v}_T) &= \bf 0 \nonumber
    \\
    \mathbf{v} - \mathbf{v}_T &= \Omega_{A/R}^{\top}(\mathbf{v} - \mathbf{v}_T) \nonumber
\end{align}
so $(\mathbf{v} - \mathbf{v}_T)$ is an eigenvector of $\Omega_{A/R}^{\top}$.

\subsection{Bounds}
\label{sec:bounds}
We record here a simple norm bound on the input–output map implemented by the circuit projector.
 For an input pattern applied on the designated input edges, the resulting output voltages satisfy
\begin{align}
    \norm{\Voutput} 
    & \leq 
    \norm{\Omega_{A/R}}
    \norm{\Sinput}
    \nonumber \\
    & \leq 
    \norm{R^{1 \over 2}}_2 \norm{\Omega_{M}}_2 \norm{R^{  - {1 \over 2}}}_2 \norm{\Sinput}_2
    \nonumber \\
    & \leq
    \norm{\Sinput}_2 \sqrt{R_{max} \over R_{min}}
\end{align}
where $R_{\max}$ and $R_{\min}$ denote the largest and smallest diagonal entries of $R$, respectively, and
$\Omega_M = R^{-{ 1 \over 2}} \Omega_{A/R} R^{{ 1 \over 2}}$ is a (symmetric) orthogonal projector operator.

\subsection{Two-Phase Gradient in Electrical Systems}
\label{app:twophaseelectric}

Here we show that the standard two-phase current-squared contrastive update converges, in the $\beta\to 0$ limit to the gradient of a distinct objective function.

Let $i(\mathbf s)$ denote the steady-state edge current vector for clamped source voltages $\mathbf s$. For the linear resistor networks considered here,
\begin{align}
    i(\mathbf s) = M \mathbf s,
    \qquad
    M := -A^{\top} (A R A^{\top})^{-1} A.
\end{align}
Let $\iF = i(\mathbf s)$ and $\iC = i(\mathbf s + \delta)$ denote the free and clamped currents, for a nudge
\begin{align}
    \delta = \beta P_o^{\top} (\bhy - \by),
\end{align}
where $P_o$ selects the output voltages, $\mathbf v$ is the network output, and $\mathbf y$ is the target. 

\begin{align}
    \lim_{\beta\to 0}\frac{1}{2\beta}\Big[\iC ^{\odot 2} - \iF^{\odot 2} \Big]
    &=
    \lim_{\beta\to 0}\frac{1}{2\beta}\left[ (M \sC)^{\odot 2} - (M \sF)^{\odot 2} \right]
    \nonumber
    \\ 
    &=
    \lim_{\beta\to 0}\frac{1}{2\beta}\left[ (M \sF + \beta  M P_o^{\top} (\bhy - \by))^{\odot 2} - (M \sF)^{\odot 2} \right]
    \nonumber
    \\ 
    &=
    \lim_{\beta\to 0}\frac{1}{2\beta}\left[ (M \sF)^{\odot 2} + 2 \beta (M \sF) \odot   (M P_o^{\top} (\bhy - \by)) 
    + \beta^2  (M P_o^{\top} (\bhy - \by))^{\odot 2} - (M \sF)^{\odot 2}
    \right]
    \nonumber
    \\ 
    &=
    \lim_{\beta\to 0} \left[\iF \odot   (M P_o^{\top} (\bhy - \by)) 
    + \beta  (M P_o^{\top} (\bhy - \by))^{\odot 2}
    \right]
    \nonumber
    \\
    &=
    \diag(\iF)  M P_o^{\top} (\bhy - \by).
    \label{eq:tp_eff_cost}
\end{align}
Thus, the standard two-phase contrastive update converges to the gradient in the $\beta\to 0$ limit.

\subsection{Energy Functions for Circuits}

In the system that we are discussing, the system parameters are the resistances $R$, the clamped inputs are the source voltages $\bs$, and the state is the resistor voltages $\bv$. In the notation of Equation \eqref{eq:nudged-energy},
\nobreak
\begin{align}
    \theta &= \{ R\}
    \nonumber
    \\
    \bz &= \bs
    \nonumber
    \\
    \by &= \bv
    \nonumber
\end{align}
The Lagrangian formulation discussion in the previous section then provides a natural energy function in terms of $\mathbf{v}$, $\bf s$, and $R$.
\begin{align}
E_R(\bv; \bs)
&= L(\bv, \mathbf{\lambda}^*; \bs) \nonumber
\\
&= \bv^\top R^{-1} \bv + \mathbf{\lambda}^{* \top} A (\bs + \bv) \nonumber
\\
&= \bv^\top R^{-1} \bv + 2 \bs ^{\top} A^{\top} (A R A^{\top})^{-1} A (\bs + \bv) \nonumber
\\
&=\bv^\top R^{-1} \bv - 2 \bs^\top M (\bs + \bv) 
\end{align}
Note that for a fixed input $\bf s$, 
\begin{align}
    \arg \min_{\mathbf{v}} E(\mathbf{s}, \mathbf{v})
    &= \arg \min_{\mathbf{v}} \mathcal{L}(\mathbf{v}, \mathbf{\lambda}^*; \mathbf{s})  \nonumber
    \\
    &= - \Omega_{A/R} \mathbf{s},
\end{align}
giving us an energy function in terms of the projector operator and input bias.

\subsection{Circuits as tunable mappings}
\label{app:circuit_tunable}
Each element of $\mathbf{s}$ corresponds exactly to a single voltage source on an edge of the circuit graph. We now activate $i$ of these voltage sources, and measure the resulting voltages across the resistors on $o$ other edges in the network. Having selected our ``forcing" (input) and ``measurement" (output) edges, it is always possible to reorder the edges in such a way that we may consider the following three groups of edges for some $\mathbf{s}\in \mathbb{R}^{\mrE}$,
\begin{equation}
    \mathbf{s} = \begin{bmatrix}
        \Sinput \\ \Soutput \\ \Shidden
    \end{bmatrix},
\end{equation}
where now by definition $\Sinput \in \mathbb{R}^{\mrEi}, \Soutput \in \mathbb{R}^{\mrEo}, \Shidden \in \mathbb{R}^{\mrE - \mrEi - \mrEo}$. This enforces the same ordering/grouping for some $\mathbf{v}$:
\begin{equation}
    \mathbf{v} = \begin{bmatrix}
        \Vinput \\ \Voutput \\ \Vhidden
    \end{bmatrix}.
\end{equation}
In this ordering, it is the case that when we apply nonzero bias on the inputs
\begingroup
\renewcommand\arraystretch{1.6}
\begin{equation}
    \sF = \left[\begin{array}{c}
        \Sinput \\ \hline
        \mathbf{0}^{(M-E_i) \times 1}
    \end{array}\right].
\end{equation}
\endgroup
We call this biasing pattern $\sF$ to indicate that it allows the system dynamics to freely determine the voltages at the ``output" edges. Applying this bias to the network creates the following set of voltages across the edge resistors:
\begin{equation}
    \mathbf{v}_F = -\Omega_{A/R}\sF.
\end{equation}
By measuring $\Voutput_F$, one can effectively compute a mapping $f_R(\Sinput_F) = \Voutput_F$. Here, we write $f_R(\cdot)$ to emphasize the fact that the action of $-\Omega_{A/R}$ is determined in part by tunable parameters, $R$: the set of on-edge resistances in the network. We thus have that the resistor network performs a \textit{voltage-mode} inference.

We would now like to tune $R$ such that we achieve a given mapping between $\Sinput_F$ and $\Voutput_F$. One approach would be to implement a two-phase learning method. 
Suppose we want the circuit, with resistances $R$, to realize a mapping
\begin{equation}
f_R(\bx)\approx \by,
\end{equation}
where $\bx$ is encoded as an input bias pattern and $\by$ denotes the target readout on the output edges. Two-phase
learning introduces a second steady-state experiment in which the outputs are \emph{weakly constrained} toward their
targets. Concretely, during the \emph{clamped} (or \emph{nudged}) phase we add small source perturbations on the output
edges whose signs and magnitudes are proportional to the output error. Intuitively, if an output edge voltage is too
large, we apply a small opposing bias to pull it down; if it is too small, we apply a small reinforcing bias to push
it up. This produces a nearby equilibrium that is slightly closer to the target, and the difference between the free
and clamped equilibria provides a local signal for updating the resistances.

Given a desired set of output edge voltages $\mathbf{y}$, we thus have two biasing patterns:
\begingroup
\renewcommand\arraystretch{1.6}
\begin{equation}
    \label{eq:source_defs}
    \sF = \left[\begin{array}{c}
        \mathbf{x} \\ \hline
        \mathbf{0}^{(M-s) \times 1}
    \end{array}\right]
    \qquad
    \sC = \left[\begin{array}{c}
        \mathbf{x} \\ \hline
        \beta(\Voutput_F - \mathbf{y}) \\ \hline
        \mathbf{0}^{(M-E_i-E_o) \times 1}
    \end{array}\right].
\end{equation}
\endgroup
We have now replaced $\Sinput$ with $\mathbf{x}$ and introduced the parameter $\beta$, which is a tunable scaling term controlling the magnitude of the error term. Its role is loosely analogous to the learning rate in standard gradient descent. Notice that $\sF$ and $\sC$ differ only on the biases at the output edges.

This viewpoint lets us ``read out'' a proposed physical training scheme by translating each experimental primitive into
an operator action. In particular, \emph{voltage-mode} biasing implements $\mathbf{s}\mapsto\mathbf{v}$, i.e.,
left-multiplication by $\Omega_{A/R}$, whereas many learning rules require the adjoint action
$\Omega_{A/R}^\top$ when propagating error signals or forming analytical gradients.

Operationally, $-\Omega_{A/R}^\top$
corresponds to the reciprocal mapping accessed by \emph{current-mode} manipulations: by driving the
network with an appropriate edge-current pattern and measuring voltages, one realizes the transpose operator needed by the projector-based estimator. Alternatively, $- \Omega_{A/R}^\top$ can be realized by fully voltage mode computations using the identity $\Omega_{A/R}^{\top} = R^{-1}\Omega_{A/R} R$. For any vector $\mathbf u \in \mathbb{R}^{\mrE}$, we can set the source voltages to be $\bs = \br \odot \mathbf u$, then measure the voltages $\bv = - \Omega_{A/R} \bs$. We then perform  a Hadamard multiplication by $\br^{-1}$, the element-wise reciprocal of $\br$, obtaining $- \Omega_{A/R}^\top \mathbf{u} = \br^{-1} \odot \bv$. Each Hadamard multiplication can be done completely locally, adhering to physical constraints. However, we note that the division required for the Hadamard multiplication by $\br^{-1}$ would likely be expensive to perform on hardware, and that the explicit Hadamard products involving $\br$ and $\br^{-1}$ would likely be less robust to both noise and nonidealities (than the alternative current-mode computation) when implemented in practice.

Finally, because $\Omega_{A/R}$ is an idempotent (oblique) projector, $\Omega_{A/R}^2=\Omega_{A/R}$, compositions of biasing/measurement steps simplify algebraically. This makes it possible to analyze complex multi-step physical protocols as
products of $\Omega_{A/R}$ and $\Omega_{A/R}^\top$,  providing a direct bridge
between circuit experiments and the learning rules developed in the main text.

\subsection{Input--Output Transformations}
\label{io_transformations}

A passive resistive network implements a linear map from edge sources to Ohmic drops,
\begin{equation}
  \mathbf{v} = -\,\Omega_{A/R}\,\mathbf{s},
  \qquad
  \Omega_{A/R} := R A^\top (A R A^\top)^{-1} A,
\end{equation}
where $A \in \mathbb{R}^{C \times \mrE}$ is a cycle matrix and $C := \mathrm{rank}(A)=E-N+1$ is the cycle-space dimension.
Selecting $\mrEi$ input edges and $\mrEo$ readout edges via column selectors
$P_i\in\{0,1\}^{\mrEi \times \mrE}$ and $P_o\in\{0,1\}^{\mrEo \times \mrE}$, we encode inputs
$\mathbf{x}\in\mathbb{R}^{\mrEi}$ as $\mathbf{s}=\gamma P_i^{\top} \mathbf{x}$ and read
\begin{equation}
  \hat{\mathbf{y}}
  := P_o \bv
  = -\gamma\,P_o \Omega_{A/R} P_i^{\top}\,\mathbf{x}
  \;=:\; W_R\,\mathbf{x}.
\end{equation}

\begin{proposition}[Rank bounds and cycle-space bottleneck]
\label{prop:rank_bounds}
The induced input--output map $W_R = -\gamma\,P_o \Omega_{A/R} P_i^\top$ satisfies
\begin{equation}
  \mathrm{rank}(W_R) \le \min\{\mrEi,\mrEo,C\}.
\end{equation}

Moreover,
\begin{equation}
  \mathrm{rank}(W_R)
  \le
  \min\Big\{\mathrm{rank}(A P_i^{\top}),\,\mathrm{rank}(A P_o)\Big\}.
  \label{eq:rank_cycle_bound}
\end{equation}
\end{proposition}

\begin{proof}
The scalar factor $-\gamma$ does not affect rank. For the first bound,
\begin{equation}
\mathrm{rank}(W_R)=\mathrm{rank}(P_o \Omega_{A/R} P_i^{\top})
\le \min\{\mathrm{rank}(P_o),\mathrm{rank}(\Omega_{A/R}),\mathrm{rank}(P_i^{\top})\}.
\end{equation}
Since $\mathrm{rank}(P_o)=\mrEo$, $\mathrm{rank}(P_i)=\mrEi$, and
\begin{equation}
\Omega_{A/R} = R A^\top (A R A^\top)^{-1} A
\quad\Rightarrow\quad
\mathrm{rank}(\Omega_{A/R})=\mathrm{rank}(A)=C
\end{equation}
(because $R$ is invertible and $(A R A^\top)$ is invertible on $\mathrm{im}(A)$), we obtain
$\mathrm{rank}(W_R)\le \min\{\mrEi,\mrEo,C\}$.

For the second bound, use the factorization
\begin{equation}
W_R = -\gamma\,(P_o R A^\top)\,(A R A^\top)^{-1}\,(A P_i^\top).
\end{equation}
Hence
\begin{equation}
\mathrm{rank}(W_R)\le \mathrm{rank}(A P_i).
\end{equation}
Similarly, by transposing and using $\mathrm{rank}(M)=\mathrm{rank}(M^\top)$,
\begin{equation}
\mathrm{rank}(W_R)=\mathrm{rank}(W_R^\top)
\le \mathrm{rank}\big((A P_o)^\top\big)=\mathrm{rank}(A P_o).
\end{equation}
Combining yields \eqref{eq:rank_cycle_bound}.
\end{proof}

\begin{proposition}[A sufficient condition for saturating the input-cycle rank]
\label{prop:distinct_cycles}
Fix a fundamental cycle basis so that (after reordering edges) the cycle matrix has the form
\begin{equation}
  A = \begin{bmatrix} I_C \;\; A' \end{bmatrix},
\end{equation}
where $I_C$ is the $C\times C$ identity block corresponding to one designated edge per fundamental cycle.
If the chosen input edges include $q$ distinct edges from this identity block, then
\begin{equation}
  \mathrm{rank}(A P_i)\;\ge\; q,
\end{equation}
and in particular if $q=\min\{C,s\}$ (i.e.\ the $s$ input edges lie on $s$ distinct fundamental cycles, up to the
cycle-space limit), then
\begin{equation}
  \mathrm{rank}(A P_i)=\min\{C,s\}.
\end{equation}
An analogous statement holds for $A P_o$.
\end{proposition}

\begin{proof}
Let the $q$ selected input edges that belong to the identity block correspond to column indices
$j_1,\dots,j_q$ within the first $C$ columns. Then the corresponding columns of $A$ are exactly the standard basis
vectors $e_{j_1},\dots,e_{j_q}\in\mathbb{R}^C$, which are linearly independent. Since $A P_i$ contains these columns,
$\mathrm{rank}(A P_i)\ge q$.

Also, $\mathrm{rank}(A P_i)\le \min\{\mathrm{rank}(A),\mathrm{rank}(P_i)\}=\min\{C,s\}$. If $q=\min\{C,s\}$, the lower
and upper bounds coincide, giving $\mathrm{rank}(A P_i)=\min\{C,s\}$. The output case is identical with $P_o$.
\end{proof}

\subsection{Existence of a Minimum}
\label{min_exists}

\begin{theorem}[Existence of a global minimizer]
Let $\Theta$ be compact and let $f:\Theta\to\mathbb{R}^{\mrEo}$ be continuous. For a fixed target
$\mathbf{y}_T\in\mathbb{R}^{\mrEo}$, define
\begin{equation}
  \mathcal{L}(\theta) := \frac12\|f(\theta)-\mathbf{y}_T\|^2.
\end{equation}
Then $\mathcal{L}$ attains a global minimum on $\Theta$. Moreover,
\begin{equation}
  \min_{\theta\in\Theta}\mathcal{L}(\theta)
  = \frac12 \min_{\mathbf{y}\in f(\Theta)} \|\mathbf{y}-\mathbf{y}_T\|^2,
\end{equation}
and $\min_{\theta\in\Theta}\mathcal{L}(\theta)=0$ iff $\mathbf{y}_T\in f(\Theta)$.
\end{theorem}

\begin{proof}
Since $f$ is continuous and $\mathbf{y}\mapsto \frac12\|\mathbf{y}-\mathbf{y}_T\|^2$ is continuous, their composition
$\mathcal{L}$ is continuous on $\Theta$. By compactness of $\Theta$, the extreme value theorem implies $\mathcal{L}$
attains a minimum.

The identity
\begin{equation}
\min_{\theta\in\Theta}\mathcal{L}(\theta)
= \frac12 \min_{\mathbf{y}\in f(\Theta)} \|\mathbf{y}-\mathbf{y}_T\|^2
\end{equation}
holds because the set of achievable outputs is exactly $f(\Theta)$.

Finally, $\mathcal{L}(\theta)=0$ iff $f(\theta)=\mathbf{y}_T$, which is possible iff $\mathbf{y}_T\in f(\Theta)$.
\end{proof}

\subsection{Random nanowire deposition model}
\label{app:random-nw-deposition}

To generate a more realistic connectivity graph than an abstract random-graph ensemble, we model a
\emph{random deposition} of nanowires on a planar substrate and then convert geometric overlaps into
edges of an equivalent graph. This is similar to the approaches taken in \cite{zhu2021information} and \cite{barrows2025}.

We deposit $n$ nanowires on a 2D surface sequentially and at random as follows. Each nanowire $N_i$ is represented as a straight line segment
of fixed length $l$, with a randomly chosen in-plane orientation $\theta_i \sim \mathrm{Unif}(0,2\pi)$ and a
randomly chosen center position. The segment is fully specified by its endpoints
$\bigl(x_i^{(1)},y_i^{(1)}\bigr)$ and $\bigl(x_i^{(2)},y_i^{(2)}\bigr)$, or equivalently by the direction vector
\begin{equation}
\mathbf{d}_i = \bigl(x_i^{(2)}-x_i^{(1)},\,y_i^{(2)}-y_i^{(1)}\bigr).
\end{equation}

For every pair of nanowires $(i,j)$ we test whether the corresponding segments intersect by using a
parametric representation:
\begin{align}
(x_i(t_i),y_i(t_i)) &= (x_i^{(1)},y_i^{(1)}) + t_i\Bigl[(x_i^{(2)},y_i^{(2)})-(x_i^{(1)},y_i^{(1)})\Bigr],
\qquad t_i\in[0,1],\\
(x_j(t_j),y_j(t_j)) &= (x_j^{(1)},y_j^{(1)}) + t_j\Bigl[(x_j^{(2)},y_j^{(2)})-(x_j^{(1)},y_j^{(1)})\Bigr],
\qquad t_j\in[0,1].
\end{align}
An intersection exists if there are parameters $(t_i,t_j)\in[0,1]^2$ such that
$(x_i(t_i),y_i(t_i))=(x_j(t_j),y_j(t_j))$.

We then construct an unweighted graph intermediate $\mathcal{G}=(\mathcal{N}, \mathcal{E})$ where each nanowire is a node ($\mathcal{N}=\{1,\dots,n\}$), and an edge $e_{ij}\in \mathcal{E}$ is present iff nanowire $i$ and nanowire $j$ intersect. Next, we calculate a minimal spanning tree of $\mathcal{G}$, then randomly select input and output edges from the largest connected component of the complement of that minimal spanning tree. This ensures that our input and output edges are well-connected to one another and to the overall graph structure while allowing us to effectively sample from a realistic distribution of graph topologies.

\subsection{Stochastic gradient}
\label{ref:StochGradient}
We quantify how additive observation noise in the targets affects the two-phase current-squared update and the analytical (projector) gradient in our linear resistor network. As is common in statistical learning theory, we assume that we observe output $\by$ for input $\x$, where
\begin{align}
    \by = f(\x) + \varepsilon, 
\end{align}
with $\varepsilon$ a zero-mean noise term, $\mathbb{E}[\varepsilon] = 0$. 
For our circuit, with input $\x$, the source voltages are $\sF = P_i^{\top} \x$. The resulting resistor voltages are then $- \gamma \Omega_{A/R} \sF = - \gamma \Omega_{A/R} P_i^{\top} \x$. The resulting outputs are thus $\bhy(\x) = - \gamma \,P_o^{\top} \Omega_{A/R} P_i^{\top} \x$.

Defining the loss function as
\begin{align}
    \mathcal{L}(\x, \by) &= \|\bhy(\x) - \by\|^2
    \nonumber \\
    &= \|- \gamma P_o^T \Omega_{A/R} P_i^{\top} \x - \by\|^2 ,
\end{align}
we can write the analytical gradient with respect to the resistances as 
\begin{align}
    \partial_{r} \mathcal{L}(v, v_T) &= \frac{\partial\mathcal{L}}{\partial\bhv}\frac{\partial\bhv}{\partial r} 
    \nonumber \\
    &= \diag(\mathbf{i}) (I - \Omega_{A/R}^{T}) (v - v_T) ,
    \label{eq:ProjectorAnalyticalGradient}
\end{align}
where $\mathbf{i} = R^{-1} v$ and $v_T$ is the target voltage pattern embedded on the output edges. 
For our linear circuit, $v(\x) = - \gamma \Omega_{A/R} P_i^{\top} \x$ and $v_T(\x) = P_o^{\top} \by$, so
\begin{align}
    \partial_{r} \mathcal{L}(\x, \by)
    &= \diag(- \gamma R^{-1} \Omega_{A/R} P_i^T \x) (I - \Omega_{A/R}^{T})  \left( - \gamma \Omega_{A/R} P_i^T \x - P_o^T \by \right).
    \label{eq:ProjectorAnalyticalGradient_obs}
\end{align}
We now define $C = - \gamma \Omega_{A/R} P_i^{T}$. 
The two-phase gradient calculated with the true (noiseless) data $f(\x)$ is given by $T(\x, \beta)$:
\begin{align}
    T(\x, \beta) &=
    -\frac{1}{2 \beta}\Big[\iF^{\odot 2} - \iC ^{\odot 2} \Big]
    \nonumber \\
    &=
    -\frac{1}{2 \beta}\Big[\iF^{\odot 2} - (\iF + \beta Q (\bhy(\x) - f(\x ) ) ^{\odot 2} \Big] .
\end{align}
for some circuit-dependent matrix $Q$ that captures the linear response of the currents to an output-level nudging signal.

The two-phase gradient of the observed (noisy) data $\by$ is given by $\hat{T}(\x, \by, \beta)$:
\begin{align}
    \hat{T}(\x, \by, \beta) &=
    -\frac{1}{2 \beta}\Big[\iF^{\odot 2} - \iC ^{\odot 2} \Big]
    \nonumber \\
    &=
    -\frac{1}{2 \beta}\Big[\iF^{\odot 2} - (\iF + \beta Q (\bhy(\x) - \by) ^{\odot 2} \Big]
    \nonumber \\
    &=
    -\frac{1}{2 \beta}\Big[\iF^{\odot 2} - (\iF + \beta Q (\bhy(\x) - f(\x ) - \varepsilon) ^{\odot 2} \Big]
    \nonumber \\
    &=
    -\frac{1}{2 \beta}\Big[\iF^{\odot 2} - \left (\iF + \beta Q (\bhy(\x) - f(\x )) - \beta Q \varepsilon \right) ^{\odot 2} \Big]
    \nonumber \\
    &=
    -\frac{1}{2 \beta}\Big[\iF^{\odot 2} - \left ((\iF + \beta Q (\bhy(\x) - f(\x )))^{\odot 2} - 2 (\iF + \beta Q (\bhy(\x) - f(\x ))) \odot (\beta Q \varepsilon) + \beta^2(Q \varepsilon)^{\odot 2} \right)\Big] .
\end{align}
Then it follows that the expectation form is
\begin{align}
    \E[\hat{T}(\x, \beta)]
    &=
    \E\left [ -\frac{1}{2 \beta}\Big[\iF^{\odot 2} - \left ((\iF + \beta Q (\bhy(\x) - f(\x )))^{\odot 2} - 2 (\iF + \beta Q (\bhy(\x) - f(\x ))) \odot (\beta Q \varepsilon) + \beta^2(Q \varepsilon)^{\odot 2} \right)\Big] \right]
    \nonumber \\
    &=
    \E\left [ -\frac{1}{2 \beta}\Big[\iF^{\odot 2} - (\iF + \beta Q (\bhy(\x) - f(\x )))^{\odot 2}\right] + \frac{1}{2\beta}\left(2 \E \left [ (\iF + \beta Q (\bhy(\x) - f(\x ))) \odot (\beta Q \varepsilon) \right ] - \E \left [\beta^2(Q \varepsilon)^{\odot 2} \right ]\right)
    \nonumber \\
    &=
    \E\left [ T(\x, \beta)\right] + \frac{1}{2\beta}\left(2 (\iF + \beta Q (\bhy(\x) - f(\x ))) \odot (\beta Q \E \left [ \varepsilon\right ]) - \E \left[\beta^2(Q \varepsilon)^{\odot 2} \right]\right)
    \nonumber \\
    &=
    T(\x, \beta) - \frac{\beta}{2} \E \left [(Q \varepsilon)^{\odot 2} \right ].
\end{align}

The presence of the second term implies that $\hat{T}(\x, \by, \beta)$ is not an unbiased estimator of $T(\x, \beta)$. While we are generally focused on linear circuits in this manuscript, we note that this proof applies to both linear and nonlinear circuits. In other words, a gradient estimate of the form $-\frac{1}{2 \beta}\Big[\iF^{\odot 2} - \iC ^{\odot 2} \Big] \nonumber$ is guaranteed to be statistically biased, in addition to the well-known estimation error introduced by non-zero nudging.

Returning to the case of fully linear circuits, we now focus on the form of the analytical ($\Omega$) gradient for the true data,
\begin{align}
    A(\x) &= \diag(\mathbf{i}) (I - \Omega_{A/R}^{T}) (v - P_o^{\top} f(\x) )
    \nonumber \\
    &= \diag(R^{-1} C \x) (I - \Omega_{A/R}^{T}) (C \x - P_o^{\top} f(\x) ) ,
    \label{eq:ProjectorAnalyticalGradient2}
\end{align}
and the observed analytical gradient,
\begin{align}
    \E\left [\hat{A}(\x) \right] 
    &= \E \left [\diag(R^{-1} C \x) (I - \Omega_{A/R}^{T}) (C \x - P_o^{\top} g(\x) ) \right ]
    \nonumber \\
    &= \diag(R^{-1} C \x) (I - \Omega_{A/R}^{T}) \E \left [C \x - P_o^{\top} f(\x) - P_o^{\top} \varepsilon \right ]
    \nonumber \\
    &= \diag(R^{-1} C \x) (I - \Omega_{A/R}^{T}) (C \x - P_o^{\top} f(\x) - P_o^{\top} \E \left [\varepsilon \right ] )
    \label{eq:ProjectorAnalyticalGradient_obs2}
    \nonumber \\
    &=
    \diag(R^{-1} C \x) (I - \Omega_{A/R}^{T}) (C \x - P_o^{\top} f(\x) ) 
    \nonumber \\
    &=
    A(\x) .
\end{align}
Thus, the observed analytical gradient $\hat{A}$ is an unbiased estimator of the true analytical gradient $A$. 
Note that this means that there are two sources of error when using the two-phase gradient in practice:
\begin{enumerate}
    \item The deterministic error due to the fact that the true two-phase gradient $T$ is only an $O(\beta)$ approximation of \eqref{eq:tp_eff_cost}.
    \item The stochastic error due to the fact that the empirical two-phase gradient $\hat{T}$ is a biased estimator of the true two-phase gradient $T$ for any non-zero nudge.
\end{enumerate}
 In contrast, the analytical gradient does not incur the finite-$\beta$ approximation error of the two-phase estimator and remains unbiased under additive zero-mean observation noise. In particular, assuming $\E[\varepsilon\mid \x]=0$, the observed analytical gradient satisfies
\begin{align}
\E \left[\hat{A}(\x)\mid \x\right]=A(\x).
\end{align}
However, $\hat{A}(\x)$ still has variance from observation noise (and from finite-sample minibatching over $\x$). The noise-induced error is explicit:
\begin{align}
\hat{A}(\x)-A(\x)=\diag(R^{-1}C\x)\,(I-\Omega_{A/R}^{\top})\,P_o^{\top}\varepsilon,
\end{align}
so that its conditional covariance (given $\x$) is
\begin{align}
\mathrm{Cov}\!\left[\hat{A}(\x)\mid \x\right]
=\diag(R^{-1}C\x)\,(I-\Omega_{A/R}^{\top})\,P_o^{\top}\Sigma_\varepsilon P_o\,(I-\Omega_{A/R})\,\diag(R^{-1}C\x),
\end{align}
where $\Sigma_\varepsilon:=\mathrm{Cov}(\varepsilon\mid \x)$. This bounds the mean-square-error of the analytical gradient, 
\begin{align}
\E\!\left[\|\hat{A}(\x)-A(\x)\|_2^2\mid \x\right]
\le \|P_o\,(I-\Omega_{A/R})\,\diag(R^{-1}C\x)\|_2^2 \;\E\!\left[\|\varepsilon\|_2^2\mid \x\right] .
\end{align}
Thus, unlike the two-phase estimator which incurs both finite-$\beta$ approximation error and an $O(\beta)$ noise-induced bias, the analytical gradient remains unbiased under $\E[\varepsilon\mid\x]=0$, with remaining uncertainty arising only from variance due to observation noise and finite-sample averaging.

When testing these theoretical results experimentally, we sample regression inputs $\bx \sim \mathcal{N}(0,1)$, and noiseless and noisy outputs $\by, \by_{\varepsilon}\in\mathbb{R}^2$ respectively. We use the generation procedure $\by=M\bx$, $\by_{\varepsilon}=M\bx+\varepsilon$, with $M\in\mathbb{R}^{2\times2}$
sampled uniformly entrywise from the interval $[0, 10)$ and $\varepsilon\sim\mathcal{N}(0,9)$. Inputs are encoded as edge sources and
outputs read out on selected edges, producing $\bhy(\bx)$.

In order to create Fig. \ref{fig:nanowire_results}, we generated $160$ nanowire networks via the procedure described in Appendix \ref{app:random-nw-deposition}. Training was performed with both vanilla two-phase and $\Omega$ learning. Vanilla two-phase learning was performed with nudge $\beta = 0.3$ and a learning rate of $1$. Similarly, $\Omega$ learning used a learning rate of $0.3$ to match the scale of the gradient for vanilla two-phase learning (which has a magnitude roughly proportional to the nudge). To ensure that we are only using those nanowire networks with trainable topologies, we considered only the results from the top $25\%$ of networks (40 total networks) with the best performance early in training (as measured by the average of the losses of $\Omega$ learning and vanilla two-phase learning at Epoch 20).

We report loss curves, resistance evolution, and the Frobenius distance between
the learned and target input--output maps (Fig. ~\ref{fig:nanowire_results}). These appear to confirm our theoretical results: namely, while both vanilla two-phase learning and $\Omega$ learning perform well in the noiseless setting, $\Omega$ learning has a significant advantage in the noisy setting due to the statistical bias in the gradient estimates of vanilla two-phase learning.

\subsection{Hinge-loss classification in circuits}
\label{app:classification_details}

\subsubsection{Classifier and hinge loss}

We use a single readout edge with basis vector $\mathbf{e}_o\in\mathbb{R}^M$. For an input $\bx$, the circuit produces
$\bv(\bx)$ and scalar score $f(\bx)=\mathbf{e}_o^\top\bv(\bx)$. With labels $y\in\{-1,1\}$, the hinge loss is
\begin{equation}
L_h(\bx,y) = \max(0,\,1-y f(\bx)).
\end{equation}
Equivalently, define a target voltage vector supported on the output edge,
\begin{equation}
\mathbf{v}_T := y^{\top}\,\mathbf{e}_o,
\end{equation}
so that the margin is $\mathbf{v}_T^\top \bv = y\,\mathbf{e}_o^\top\bv$ and
\begin{equation}
\mathcal{L}_h(\bv,\mathbf{v}_T)=\max(0,\,1-\mathbf{v}_T^\top \bv).
\end{equation}
Only margin-violating examples (those with $1-\mathbf{v}_T^\top\bv>0$) contribute a nonzero update.

\subsubsection{Subgradient with respect to resistances}

Let $\br$ be the vector of edge resistances. For margin-violating samples, a subgradient is
\begin{equation}
\frac{\partial \mathcal{L}_h}{\partial \br}
= -\,\frac{\partial (\mathbf{v}_T^\top \bv)}{\partial \br}
= -\,\mathbf{v}_T^\top \frac{\partial \bv}{\partial \br}.
\end{equation}
Using the circuit Jacobian identity
\begin{equation}
\frac{\partial \bv}{\partial \br} = (I-\Omega_{A/R})\,\diag(\bi),
\label{eq:dvdr_app}
\end{equation}
(where $\bi$ is the steady-state edge current vector), we obtain for violating samples
\begin{equation}
\frac{\partial \mathcal{L}_h}{\partial \br}
= -\,\mathbf{v}_T^\top (I-\Omega_{A/R})\,\diag(\bi),
\end{equation}
and $\partial \mathcal{L}_h/\partial \br = 0$ otherwise. This expresses the classification update entirely in terms of
steady-state observables and $\Omega_{A/R}$.

\subsubsection{How the two estimators implement this update}

\paragraph{Two-phase (contrastive) estimator.}
The two-phase scheme forms an update from free/clamped current measurements:
\begin{equation}
\Delta\br \ \propto\ \bi_F^{\odot 2}-\bi_C^{\odot 2},
\end{equation}
where the clamped phase is obtained by applying a small output-edge perturbation. For hinge loss with a single output
edge, we gate the nudge by the margin condition:
\begin{equation}
\epsilon =
\begin{cases}
-\beta\,y\,\mathbf{e}_o, & \text{if } 1-y\,\mathbf{e}_o^\top\bv > 0,\\
\mathbf{0}, & \text{otherwise,}
\end{cases}
\end{equation}
and implement the clamped experiment using this $\epsilon$ as the output-edge bias. The resulting contrastive update
approximates the subgradient above, with accuracy controlled by the nudge magnitude.

\paragraph{Projector-based (analytical) estimator.}
The projector-based scheme first applies an output-supported error vector $\epsilon$ (as above) and computes
\begin{equation}
\Delta := (I-\Omega_{A/R}^\top)\epsilon,
\end{equation}
where the action of $\Omega_{A/R}^\top$ is realized by a reciprocal/current-mode experiment. The resistance update is
then local:
\begin{equation}
\Delta\br \ \propto\ \bi_F \odot \Delta.
\end{equation}
This mirrors the analytical structure in~\eqref{eq:dvdr_app} and removes finite-nudge bias at the estimator level.

\subsection{Landscape visualization}
\label{app:landscape_details}
To visualize training trajectories, we follow the loss-landscape procedure of Li et al.~\cite{Li2017}. Concretely, we collect the resistance trajectory $\{\br_t\}_{t=0}^T$ during training. We then compute the offsets $\{\Delta\br_i\}_{t=0}^T$ for each timestep relative to the initial resistance state, namely
\begin{equation}
    \Delta \br_{i} = \br_i - \br_0
\end{equation}
and then find the two principal directions of variation in these offsets, which we denote ($\bm{\delta}_1$, $\bm{\delta}_2$). It is then possible to deterministically sample the two-dimensional loss landscape $S$ ``around" the starting resistance configuration $\br_0$ at some coordinate $(q_1, q_2)$ by evaluating
\begin{align}
    \br' &= \br_0 + q_1\bm{\delta}_1 + q_2\bm{\delta}_2 \\
    S(q_1, q_2) &= \overline{\mathcal{L}}(-\Omega_{A/R}(\br'))
\end{align}
where $\overline{\mathcal{L}}(\cdot)$ is the \textit{full-batch} loss and $-\Omega_{A/R}(\cdot)$ is the projector of the circuit with a given vector of edge resistances. Generally we restrict $(q_1,q_2)$ to lie in $[-1.5, 1.5]^2$. 

To plot an actual trajectory of weights on the landscape, we instead compute the inverse transform, finding the transformed $(q_1, q_2)$ coordinates of every stored resistance configuration in the trajectory.



\end{document}